\RequirePackage{fix-cm}

\documentclass[smallextended]{svjour3}

\providecommand{\smartqed}{}
\smartqed

\AtBeginDocument{\hfuzz=200pt\vfuzz=200pt\hbadness=10000\vbadness=10000}
\usepackage{amsmath, amssymb}
\usepackage{booktabs}
\usepackage{multirow}
\usepackage{graphicx}
\usepackage{xcolor}
\usepackage[protrusion=true,expansion=false]{microtype}
\usepackage[shortlabels]{enumitem}
\usepackage{algorithm}
\usepackage{algpseudocode}
\usepackage{natbib}
\usepackage{url}
\usepackage{threeparttable}
\usepackage{hyperref}
\usepackage[RPvoltages, siunitx]{circuitikz}
\usepackage{placeins}

\usepackage{pgfplots}
\usepackage{pgfplotstable}
\pgfplotsset{compat=1.18}
\usepgfplotslibrary{groupplots}
\usetikzlibrary{arrows.meta, positioning, fit, backgrounds, shapes.geometric,
                decorations.pathmorphing, calc}
 
\definecolor{C1}{HTML}{1A73E8}  \definecolor{C1L}{HTML}{D2E3FC}
\definecolor{C2}{HTML}{7B2FBE}  \definecolor{C2L}{HTML}{EDE7F6}
\definecolor{C3}{HTML}{37474F}  \definecolor{C3L}{HTML}{ECEFF1}
\definecolor{C4}{HTML}{D56E0C}  \definecolor{C4L}{HTML}{FFF3E0}
\definecolor{C5}{HTML}{007B83}  \definecolor{C5L}{HTML}{E0F5F5}
\definecolor{Grn}{HTML}{1E8E3E} \definecolor{GrnL}{HTML}{E6F4EA}
\definecolor{Red}{HTML}{C5221F} \definecolor{RedL}{HTML}{FCE8E6}
\definecolor{Amb}{HTML}{E37400} \definecolor{AmbL}{HTML}{FEF7E0}
\definecolor{Gry}{HTML}{9E9E9E} \definecolor{GryL}{HTML}{F5F5F5}
 
\tikzset{
  badge/.style = {draw=#1, fill=#1, rounded corners=2.5pt,
                  font=\sffamily\bfseries\scriptsize, text=white,
                  inner xsep=5pt, inner ysep=2.5pt,
                  minimum width=11.80cm},
  rbox/.style  = {draw=#1, fill=#1!12, rounded corners=3pt,
                  inner sep=5pt, font=\sffamily\scriptsize},
  arr/.style   = {-Stealth, #1, line width=0.9pt},
  sarr/.style  = {-Stealth, #1, line width=0.65pt},
}
 
\newcommand{\PASS}{\textcolor{Grn}{\bfseries PASS}}

\newcommand{\CK}{\textcolor{Grn}{\bfseries\checkmark}\,}
\newcommand{\XX}{\textcolor{Red}{\bfseries$\times$}\,}

\definecolor{minilmcolor}{RGB}{130, 120, 180}   
\definecolor{e5color}{RGB}{220, 120, 40}         
\providecommand{\journalname}[1]{}
\journalname{Language Resources and Evaluation}

\newcommand{\credence}{Credence}
\newcommand{\semf}{\text{Semantic-F1}}
\newcommand{\epr}{\text{EPR}}
\newcommand{\avr}{\text{AVR}}
\newcommand{\rrate}{\text{RR}}
\newcommand{\socialclaimsplit}{\textsc{SocialClaimSplit}}
\newcommand{\wikisplitbench}{\textsc{WikiSplitBench}}

\graphicspath{{resources/}{figures/}}
\AtBeginDocument{\pdfstringdefDisableCommands{\def\tau{tau}\def\mathrm#1{#1}\def\_#1{#1}}}

\begin{document}
\hfuzz=200pt\hbadness=10000\vbadness=10000\sloppy

\title{CREDENCE: Claim Reduction for Decomposition \&
       Enhanced Credibility}
\subtitle{Semantic Metrics and Convergence Analysis}


\author{Phuong~Huu~Vu~Tran \and Thuan~Duc~Mai \and Bach~Xuan~Le\thanks{Corresponding author: lexuanbach@hcmut.edu.vn}}

\authorrunning{P. H. V. Tran, T. D. Mai, and B. X. Le}

\institute{
  Phuong~Huu~Vu~Tran \at
    Vietnamese-German University, Ho Chi Minh City, Vietnam \\
    \email{10425032@student.vgu.edu.vn} \\
    ORCID: \url{https://orcid.org/0009-0003-9433-473X}
  \and
  Thuan~Duc~Mai \at
    Ho Chi Minh University of Technology, Ho Chi Minh City, Vietnam \\
    \email{thuan.maiduc2411@hcmut.edu.vn} \\
    ORCID: \url{https://orcid.org/0009-0000-9242-6874}
  \and
  Bach~Xuan~Le \at
    Ho Chi Minh University of Technology, Ho Chi Minh City, Vietnam \\
    \email{lexuanbach@hcmut.edu.vn} \\
    ORCID: \url{https://orcid.org/0009-0003-6848-5403}
}

\date{Received: April, 2026 / Accepted: ---}

\maketitle

\begin{abstract}
Decomposing compound sentences into atomic, verifiable claims
is a prerequisite for reliable automated fact-checking. Prior work has
relied on token-overlap (Jaccard) metrics that systematically
underestimate decomposition quality for paraphrastic claims, and
has lacked formal termination analysis for the repair loop.

We present \credence{}, a revised claim decomposition and evaluation
framework addressing both shortcomings. Our contributions are:
(1)~\textbf{Semantic-F1}: we use BGE-large cosine similarity fidelity metric that resolves Jaccard's penalisation and improves downstream fact-checking accuracy;
(2)~\textbf{Convergence theorems}: we formally characterise four properties of the repair pipeline, establishing that rule-based repair is monotone and finitely terminating under an oracle parser assumption; LLM-based self-repair is provably non-monotone and requires an early-exit guard;
(3)~\textbf{Three evaluation benchmarks} spanning social-media, encyclopaedic, and news domains for cross-domain generalisation measurement;
(4)~\textbf{Multi-model benchmarking} across four decomposer models (3.8B-12B) and a closed API model.

Experiments on SocialClaimSplit, WikiSplitBench, and ClaimDecompBench
show that Semantic-F1 outperforms Jaccard-F1 by $+$15-32pp.
EPR ranges from 0.94 to 1.00 on SocialClaimSplit and WikiSplitBench,
while ClaimDecompBench includes lower base EPR cases (down to 0.824)
due to harder news-domain constructions, and rule-repair reduces the
Atomicity Violation Rate (AVR) by 47--100\% relative to the base model
without degrading fidelity.

\end{abstract}

\section{Introduction}
\label{sec:intro}
Automated fact-checking systems face a fundamental challenge: real-world
claims are often compound sentences bundling multiple distinct assertions.
For example, \textit{``The Prime Minister announced a \$50B stimulus and
the unemployment rate fell to 3.8\%''} contains two independently
verifiable sub-claims. Directly retrieving evidence against the composite
sentence is difficult; systems that first decompose it into atomic units
achieve substantially higher verification accuracy
\citep{min2023factscore,wice2023}.

\paragraph{The measurement problem.}
Existing decomposition systems are typically evaluated with
\emph{Jaccard token overlap} (Soft-F1$_{\text{Jac}}$), which assigns zero
similarity to paraphrastic but semantically equivalent claims such as
\textit{``X is large''} and \textit{``X is big''}. This causes the metric
to systematically penalise models that correctly paraphrase while
atomising, creating a spurious gap between human judgements and automatic
scores~\citep{bertscore2020}.

\paragraph{The convergence problem.}
Existing methods typically introduce a naive repair loop --—
rule-based followed by LLM self-repair --- but provided no formal
analysis of whether the loop terminates or whether metrics improve monotonically.
In practice, LLM self-repair can introduce new compound claims while resolving old ones. This risk is consistent with the broader finding of \citep{huang2023selfcorrect}, who showed that LLM self-correction without external feedback often degrades output quality in reasoning tasks.

\paragraph{Our contributions.}
Overall, \credence{} is designed as a rigorous framework refinement and formalization for claim decomposition. It addresses both problems through four primary contributions:

\begin{enumerate}[label=\textbf{C\arabic*.}, leftmargin=2em]
  \item \textbf{Semantic-F1 metric.}
        We replace Jaccard overlap with BGE-large-en-v1.5
        cosine similarity~\citep{bge2024}, re-implemented as an
        average-max F1 (Section~\ref{sec:metrics}).
        Semantic-F1 consistently outperforms Jaccard-F1 by
        $+$15--32 percentage points absolute (mean $+$25\,pp across
        all datasets and models; Table~\ref{tab:jaccard_comparison}),
        while resolving Jaccard's critical structural flaw---the
        improper penalisation of semantic paraphrases, synonym
        substitutions, and word-order variations.
        Embedding-based metrics fundamentally evaluate meaning over
        token overlap~\citep{bertscore2020}.

  \item \textbf{Convergence theorems.}
        We formally establish four convergence properties (T1--T4)
        under Assumption~1 (oracle parser), characterising the
        repair pipeline (Section~\ref{sec:convergence}):
        rule-repair monotonically non-increases the compound-boundary
        violation count and terminates (T1, T2);
        LLM self-repair is non-monotone with a constructible
        counterexample (T3);
        under union-based EPR, $\epr=1$ and $\avr=0$ are jointly
        achievable; practical interference arises when decomposers drop
        entities during splitting (T4, Constraint Interference).

      \item \textbf{WikiSplitBench and ClaimDecompBench (external validation).}
        We build two WikiSplitBench variants (100/1000 examples)
        from WikiSplit$++$~\citep{wikisplitpp2023}---human Wikipedia editor
        sentence splits with \emph{no construction overlap} with
        \socialclaimsplit{} construction scripts or our constraint rules.
        We maintain two SocialClaimSplit variants (100/1000) for
        controlled in-domain analysis.
        We further introduce ClaimDecompBench-1000, drawn from AG News,
        providing news-domain evaluation orthogonal to Wikipedia text.
        Both benchmarks guard against heuristic inflation
        (Section~\ref{sec:datasets}).

  \item \textbf{Multi-model evaluation.}
        We benchmark Phi-3-mini (3.8B), Qwen3-8B (8B),
        Gemma-3-12b-it (12B), and Gemini~Flash (API) across four
        decomposition modes: base, rule-repair, self-repair, and full
        pipeline on SocialClaimSplit, WikiSplitBench, and ClaimDecompBench
        (Section~\ref{sec:results}).
        Key finding: a verified Qwen3-8B (all mode) matches or exceeds
        unverified Gemini Flash on the external WikiSplitBench benchmark,
        demonstrating practical value for privacy-sensitive deployments.
\end{enumerate}

\paragraph{Paper organisation.}
Section~\ref{sec:related} reviews related work.
Section~\ref{sec:framework} describes the \credence{} architecture.
Section~\ref{sec:convergence} presents the convergence theorems.
Sections~\ref{sec:experiments}--\ref{sec:results} cover experimental
setup and results.
Section~\ref{sec:discussion} discusses limitations.
Section~\ref{sec:conclusion} concludes.

\noindent\textbf{Contribution scope.}~The primary contribution of
this work is \emph{methodological}: the Semantic-F1 metric and the
verified repair framework.
WikiSplitBench and ClaimDecompBench are secondary resource
contributions, released to enable reproducibility and cross-domain
evaluation.

\section{Related Work}
\label{sec:related}
\subsection{Claim Decomposition}

Decomposing multi-claim sentences into atomic units is an active research
area. \citet{min2023factscore} introduced \textsc{FActScore}, which evaluates long-form generation at the level of atomic facts and motivates fine-grained factual assessment. \citet{wice2023} created the WiCE dataset of
entailment-annotated claims and showed that fine-grained decomposition
improves NLI-based verification. More recent work by \citet{chen2023sub}
proposed contrastive learning of propositional semantic representations at the sub-sentence level.
More recently, \citet{lu2025dydecomp} (DyDecomp) introduced a
reinforcement learning (PPO) approach that optimizes the atomicity
level of decompositions to maximize verification confidence, showing
that optimal granularity is verifier-dependent.
\citet{magomere2026dad} (DAD) further align a decomposer's output
to a verifier via GRPO reward shaping, improving downstream
fact-checking macro-F1 by up to $+6.24$ points.
Both systems train specialized decomposers; \credence{} is
complementary---a reference-free evaluation and repair layer
applicable post-hoc to any decomposer, including DyDecomp or DAD.

\subsection{LLM-Based Rewriting and Repair}

Self-consistency and self-correction in LLMs have been studied
by~\citet{huang2023selfcorrect}, who showed that without external
feedback, self-correction often degrades output quality --- motivating
our Theorem~3. Constitutional AI~\citep{bai2023constitutional} uses
LLM ``critics'' to revise outputs iteratively, but relies on
task-specific rubrics. \credence{} avoids open-ended iteration via a
bounded repair loop with formal termination proofs.

\subsection{Sentence Simplification and Splitting}

Sentence splitting --- transforming a complex sentence into shorter
simple sentences --- is the structural counterpart to semantic
decomposition. The WikiSplit corpus~\citep{wikisplit2018} provided
the first large-scale training set for this task. The cleaner
WikiSplit$++$~\citep{wikisplitpp2023} filtered noise using entailment
probabilities, making it suitable as an evaluation benchmark.
We use WikiSplit$++$ as the basis for our WikiSplitBench datasets.
BiSECT~\citep{kim2021bisect} proposed cross-lingual sentence compression
via bisection, highlighting the shared structure between splitting and
simplification; unlike BiSECT, \credence{} does not require parallel
training data.
\citet{chen2022claimdecomp} study decomposition of complex political claims into literal and implied subquestions for downstream fact-checking. Their task formulation is related to decomposition, but it differs from our atomic-claim splitting setting.

\subsection{Evaluation Metrics for Generation}

BERTScore~\citep{bertscore2020} computes precision/recall/F1 based on
contextualised token embeddings and correlates better with human
judgements than BLEU or ROUGE. However, it is expensive at inference
time. We use sentence-level BGE-large embeddings~\citep{bge2024} ---
averaged to a single vector per claim --- to leverage the discriminative power and robustness of embeddings at significantly lower cost.
Recent state-of-the-art evaluators have formalized decomposition evaluation with NLI-based or statistical metrics. Notably, \citet{huang2025decmetrics} proposed DecMetrics, which leverages NLI to assess information completeness, correctness, and semantic entropy. While their metric design shares goals with our EPR/AVR/RR/MC suite, \credence{}'s Semantic-F1 resolves the vulnerability of token-based scores strictly in the context of atomic granularity mismatch without penalizing valid paraphrased alignments. We further differentiate our work by providing formal convergence analysis (Theorems~\ref{thm:t1}--\ref{thm:t4}) under Assumption~1 for constraint-based repair and relying entirely on reference-free operational metrics rather than requiring explicit NLI evaluators or gold references for internal optimization. \citet{samarinas2025icat} (ICAT) evaluates factuality and coverage
of long-form outputs via decomposition followed by retrieval,
complementing our focus on atomic-level constraint satisfaction.
Benchmarks for atomic decomposition have also been developed for
Chinese long-form QA~\citep{zhang2024cacdd}, underscoring the
cross-lingual generality of the decomposition paradigm.
All significance tests use paired bootstrap resampling with $B=10{,}000$
iterations following the recommendation of~\citet{dror2018resampling}
for NLP evaluation.

\subsection{Decomposition-Based Fact-Checking}

Several recent systems use claim decomposition as a preprocessing step
for automated fact-checking.
ProgramFC~\citep{pan2023programfc} decomposes complex claims into
sub-questions, retrieves evidence per sub-question, and aggregates
verdicts programmatically.
SAFE~\citep{wei2024safe} (Search-Augmented Factuality Evaluator)
decomposes long-form model outputs into atomic facts, searches for
each via Google, and rates faithfulness individually.
ProgramFC demonstrates that fine-grained decomposition improves
fact-checking accuracy over claim-level approaches.
SAFE similarly shows that atomic decomposition improves
faithfulness evaluation of long-form outputs --- a complementary
but related finding in the factuality evaluation setting.
We quantify the AFC-side effect in Section~\ref{sec:afc_eval} using
our own decomposer as the splitting stage.
AFEV~\citep{afev2024} and FASTFACT~\citep{fastfact2024} all operate in a decompose/extract-then-verify style pipeline, but they target different settings: claim verification, long-form factuality evaluation, or iterative atomic fact extraction. In this context, \credence{} serves not as a replacement, but as a carefully controlled building block that can enforce granularity and stability within such pipelines.
Furthermore, while some studies highlight that aggressive decomposition can occasionally introduce fact-checking harms via over-fragmentation~\citep{hu2024decomposition}, \credence{} explicitly mitigates this through semantic constraint diagnostics. Unlike RL-based dynamic decomposition policies~\citep{lu2025dydecomp,
magomere2026dad}, we opt for formal verifiable limits
rather than learned heuristics.

\section{The \credence{} Framework}
\label{sec:framework}
The \credence{} pipeline takes an input
sentence~$S$ and produces a list of atomic claims $C = [c_1, \ldots, c_k]$
satisfying four desiderata:
(D1)~\emph{Atomicity}: each $c_i$ contains at most one independently
verifiable assertion;
(D2)~\emph{Fidelity}: $C$ semantically covers $S$;
(D3)~\emph{Entity preservation}: every entity mentioned in $S$ appears in
at least one $c_i$;
(D4)~\emph{Non-redundancy}: no two $c_i, c_j$ are near-paraphrases of
each other (illustrated in Figure~\ref{fig:framework}).

\setlength{\textfloatsep}{8pt plus 2pt minus 2pt}
\setlength{\floatsep}{8pt plus 2pt minus 2pt}
\setlength{\intextsep}{8pt plus 2pt minus 2pt}
\setlength{\abovedisplayskip}{6pt plus 2pt minus 2pt}
\setlength{\belowdisplayskip}{6pt plus 2pt minus 2pt}
\setlength{\abovedisplayshortskip}{4pt plus 2pt minus 2pt}
\setlength{\belowdisplayshortskip}{4pt plus 2pt minus 2pt}

\begin{figure}[ht]
    \centering
    \includegraphics[width=1\textwidth]{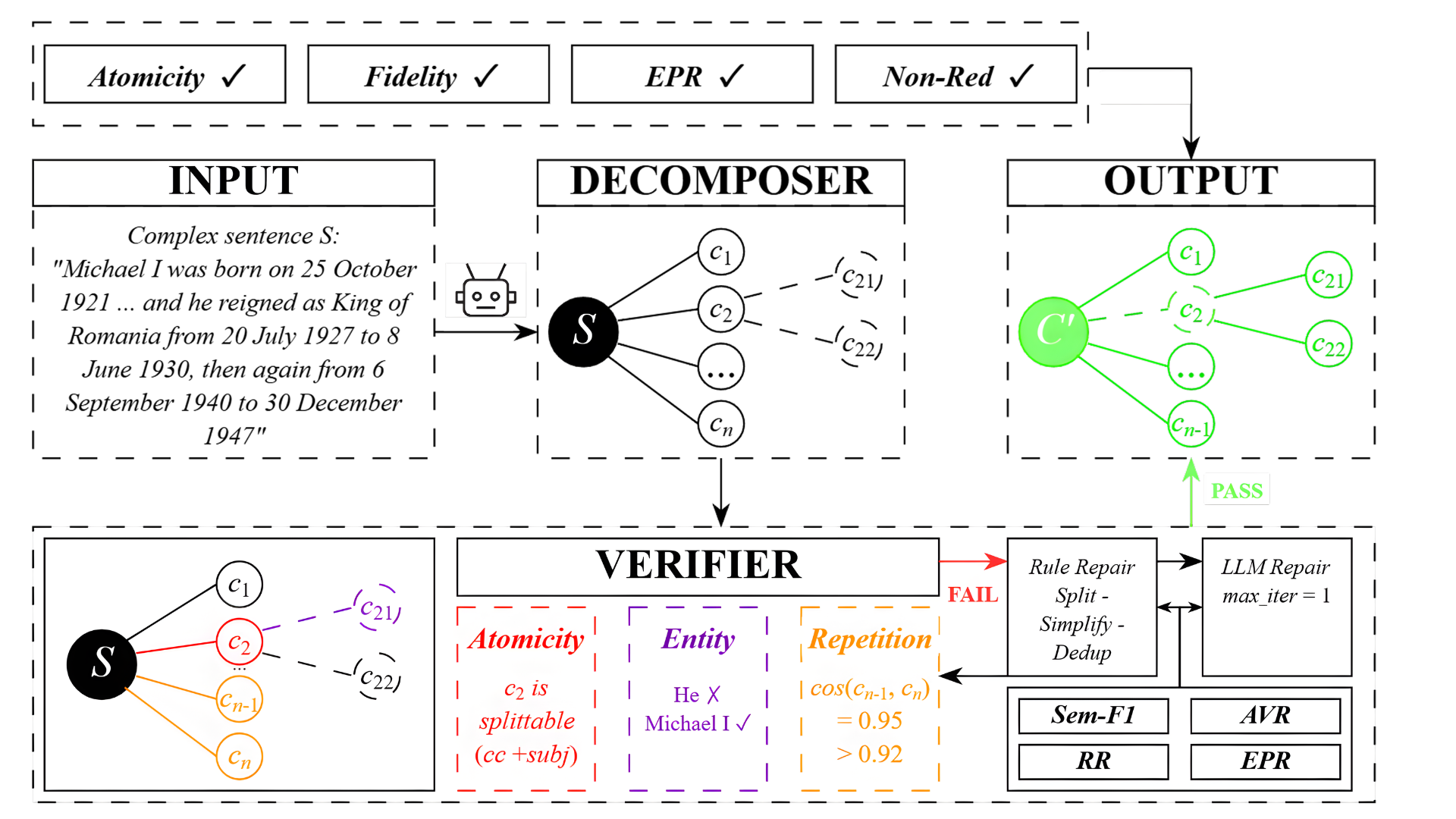}
    \caption{The CREDENCE Framework.}
    \label{fig:framework}
\end{figure}

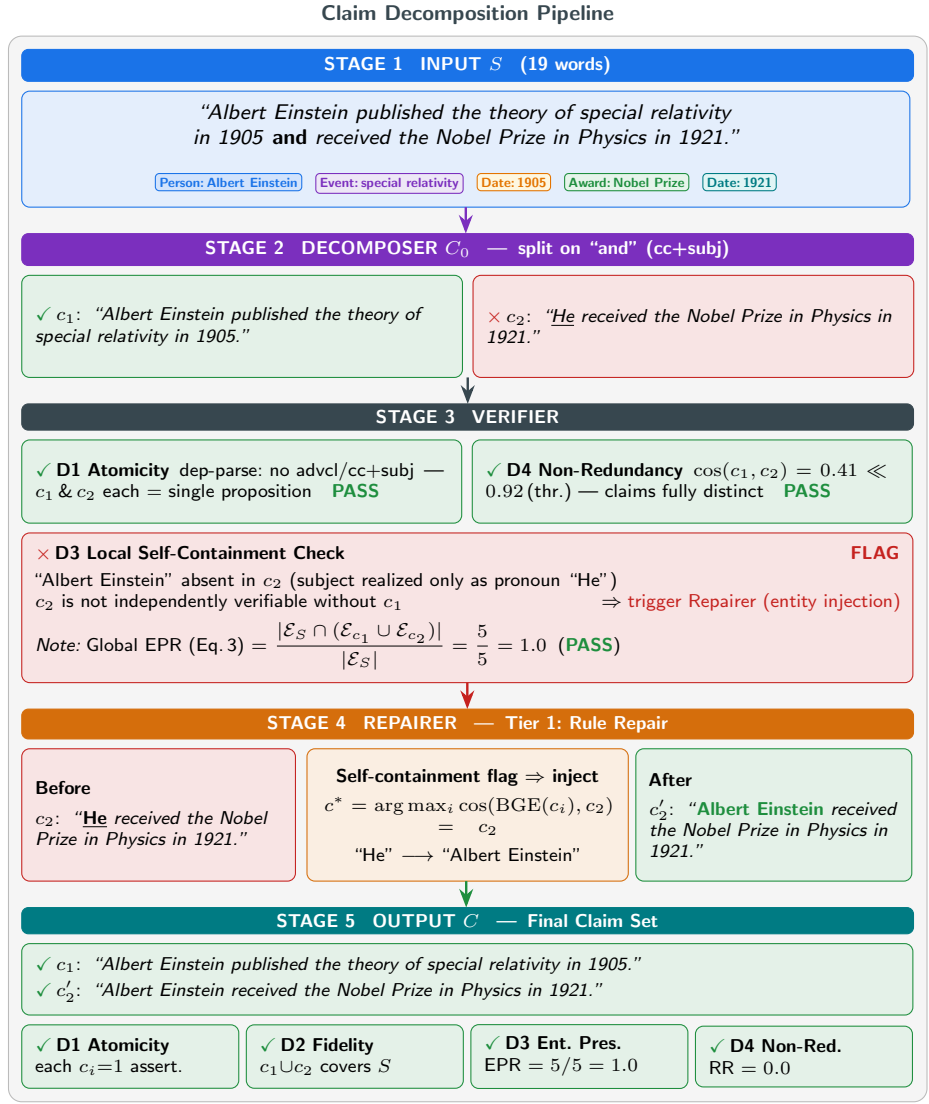
\begin{figure}[!htp]
\centering
\begin{tikzpicture}[font=\sffamily\scriptsize, node distance=0pt,
                    every node/.style={outer sep=0pt}]
 
\node[badge=C1, anchor=north west] at (0,0) (s1h)
  {STAGE 1 \enspace INPUT $S$ \enspace (19 words)};
 
\node[rbox=C1, align=center, text width=11.40cm,
      font=\sffamily\small\itshape,
      below=0.13cm of s1h.south west, anchor=north west] (s1b)
  {``Albert Einstein published the theory of special relativity in 1905
    \textbf{and} received the Nobel Prize in Physics in 1921.''\\[3pt]
   \normalfont\sffamily\scriptsize
   \tikz[baseline=-0.5ex]{\node[draw=C1,fill=C1L,rounded corners=1.5pt,
     font=\sffamily\bfseries\tiny,text=C1,inner sep=1.5pt]
     {Person:\,\textit{Albert Einstein}};}\;
   \tikz[baseline=-0.5ex]{\node[draw=C2,fill=C2L,rounded corners=1.5pt,
     font=\sffamily\bfseries\tiny,text=C2,inner sep=1.5pt]
     {Event:\,\textit{special relativity}};}\;
   \tikz[baseline=-0.5ex]{\node[draw=Amb,fill=AmbL,rounded corners=1.5pt,
     font=\sffamily\bfseries\tiny,text=Amb,inner sep=1.5pt]
     {Date:\,1905};}\;
   \tikz[baseline=-0.5ex]{\node[draw=Grn,fill=GrnL,rounded corners=1.5pt,
     font=\sffamily\bfseries\tiny,text=Grn,inner sep=1.5pt]
     {Award:\,\textit{Nobel Prize}};}\;
   \tikz[baseline=-0.5ex]{\node[draw=C5,fill=C5L,rounded corners=1.5pt,
     font=\sffamily\bfseries\tiny,text=C5,inner sep=1.5pt]
     {Date:\,1921};}
  };
 
\draw[arr=C2] (s1b.south) -- ++(0,-0.35);
 
\node[badge=C2, below=0.35cm of s1b.south west, anchor=north west] (s2h)
  {STAGE 2 \enspace DECOMPOSER $C_0$ \enspace\textmd{---}\enspace
   split on \textit{``and''} (cc+subj)};
 
\node[rbox=Grn, align=left, text width=5.48cm, minimum height=1.35cm,
      below=0.13cm of s2h.south west, anchor=north west] (c1r)
  {\CK $c_1$:\enspace\textit{``Albert Einstein published the theory
    of special relativity in 1905.''}};
 
\node[rbox=Red, align=left, text width=5.48cm, minimum height=1.35cm,
      right=0.14cm of c1r.north east, anchor=north west] (c2r)
  {\XX $c_2$:\enspace\textit{``\underline{He} received the Nobel Prize
    in Physics in 1921.''}\hfill
   };
 
 
\coordinate (bot2) at ($(c1r.south east)!0.5!(c2r.south west)$);
\draw[arr=C3] (bot2) -- ++(0,-0.35);
 
\node[badge=C3, below=0.35cm of c1r.south west, anchor=north west] (s3h)
  {STAGE 3 \enspace VERIFIER};
 
\node[rbox=Grn, align=left, text width=5.47cm, minimum height=1.10cm,
      below=0.13cm of s3h.south west, anchor=north west] (v1)
  {\CK\textbf{D1 Atomicity}\enspace
   dep-parse: no advcl/cc+subj\enspace
   — $c_1$\,\&\,$c_2$ each = single proposition\quad\PASS};
 
\node[rbox=Grn, align=left, text width=5.47cm, minimum height=1.10cm,
      right=0.14cm of v1.north east, anchor=north west] (v3)
  {\CK\textbf{D4 Non-Redundancy}\enspace
   $\cos(c_1,c_2)=0.41\ll0.92$\,(thr.)
   — claims fully distinct\quad\PASS};
 
\node[rbox=Red, align=left, text width=11.44cm,
      below=0.13cm of v1.south west, anchor=north west] (v2)
  {\XX\textbf{D3 Local Self-Containment Check}\hfill\textcolor{Red}{\bfseries FLAG}\\[2pt]
   ``Albert Einstein'' absent in $c_2$ (subject realized only as pronoun ``He'')\\
   $c_2$ is not independently verifiable without $c_1$\hfill
   $\Rightarrow$\;\textcolor{Red}{trigger Repairer (entity injection)}\\[3pt]
   \textit{Note:}\;Global EPR (Eq.\,3)
   $= \dfrac{|\mathcal{E}_S\cap(\mathcal{E}_{c_1}\cup\mathcal{E}_{c_2})|}{|\mathcal{E}_S|}
     = \dfrac{5}{5}=1.0$\;\;(\PASS)};
 
 
\draw[arr=Red] (v2.south) -- ++(0,-0.35);
 
\node[badge=C4, below=0.35cm of v2.south west, anchor=north west] (s4h)
  {STAGE 4 \enspace REPAIRER \enspace\textmd{---}\enspace Tier~1: Rule Repair};
 
\node[rbox=Red, align=left, text width=3.27cm, minimum height=1.75cm,
      below=0.13cm of s4h.south west, anchor=north west] (bef)
  {\textbf{Before}\\[3pt]
   $c_2$:\enspace\textit{``\textbf{\underline{He}} received the
   Nobel Prize in Physics in 1921.''}};
 
\node[rbox=C4, align=center, text width=3.9cm, minimum height=1.75cm,
      below=0.13cm of s4h.south, anchor=north] (inj)
  {\textbf{Self-containment flag $\Rightarrow$ inject}\\[3pt]
   $c^{*}=\arg\max_i\cos(\mathrm{BGE}(c_i),c_2)$\\$=c_2$\\[3pt]
   ``He''\;$\longrightarrow$\;``Albert Einstein''};
 
\node[rbox=Grn, align=left, text width=3.3cm, minimum height=1.75cm,
         right=0.10cm of inj.north east, anchor=north west] (aft)
  {\textbf{After}\\[3pt]
   $c_2'$:\enspace\textit{``\textbf{\textcolor{Grn}{Albert Einstein}}
   received the Nobel Prize in Physics in 1921.''}};
 
 
\coordinate (bot4) at ($(bef.south)!0.5!(aft.south)$);
\draw[arr=Grn] (bot4) -- ++(0,-0.35);
 
\node[badge=C5, below=0.35cm of bef.south west, anchor=north west] (s5h)
  {STAGE 5 \enspace OUTPUT $C$ \enspace---\enspace Final Claim Set};
 
\node[rbox=Grn, align=left, text width=11.44cm,
      below=0.13cm of s5h.south west, anchor=north west] (s5b)
  {\CK $c_1$:\enspace\textit{``Albert Einstein published the theory
    of special relativity in 1905.''}\\[2pt]
   \CK $c_2'$:\enspace\textit{``Albert Einstein received the Nobel Prize
    in Physics in 1921.''}};
 
\node[rbox=Grn, align=left, text width=2.52cm,
      below=0.13cm of s5b.south west, anchor=north west] (m1)
  {\CK\textbf{D1 Atomicity}\\each $c_i{=}1$ assert.};
\node[rbox=Grn, align=left, text width=2.52cm, right=0.10cm of m1] (m2)
  {\CK\textbf{D2 Fidelity}\\$c_1{\cup}c_2$ covers $S$};
\node[rbox=Grn, align=left, text width=2.52cm, text height = 0.15cm, right=0.10cm of m2] (m3)
  {\CK\textbf{D3 Ent.~Pres.}\\EPR $=5/5=1.0$};
\node[rbox=Grn, align=left, text width=2.52cm, right=0.10cm of m3] (m4)
  {\CK\textbf{D4 Non-Red.}\\RR $=0.0$};
 
 
\begin{scope}[on background layer]
  \node[draw=Gry!70, fill=GryL, rounded corners=5pt, inner sep=5pt,
        fit=(s1h)(s1b)(s2h)(c1r)(c2r)(s3h)(v1)(v3)(v2)
            (s4h)(bef)(inj)(aft)(s5h)(s5b)(m1)(m2)(m3)(m4)] (frame) {};
\end{scope}
\node[font=\sffamily\bfseries\small, text=C3, anchor=south,
      above=0.06cm of frame.north]
  {Claim Decomposition Pipeline};
 
\end{tikzpicture}
\caption{\sloppy End-to-end claim decomposition pipeline.
\textbf{S1}~Input $S$ with entity annotations.
\textbf{S2}~Decomposer $C_0$ splits on coordinating conjunction; $c_2$ contains
an unresolved pronoun.
\textbf{S3}~Verifier: D1 and D4 pass; global EPR (Eq.\,3) $= 1.0$ (PASS),
but a local self-containment diagnostic flags $c_2$ (unresolved pronoun),
triggering the Repairer. 
\textbf{S4}~Rule Repairer (Tier~1) injects the missing named entity into $c_2$.
\textbf{S5}~Final claim set $C$ satisfies all desiderata.}
\label{fig:pipeline}
\end{figure}

\vspace{-0.4em}
\subsection{Decomposer}
\label{sec:decomposer}

The decomposer $\mathcal{D}$ is a prompted LLM that maps $S$ to an
initial claim list $C_0$.
We evaluate four decomposer backends:
Phi-3-mini-128k (3.8B)~\citep{phi3},
Qwen3-8B~\citep{qwen3},
Gemma-3-12b~\citep{gemma3},
and Gemini~2.5~Flash~\citep{gemini25flash} (API).

All models receive the same system prompt (Appendix~A), which instructs
the model to split compound sentences, preserve all entities, and output
one claim per line without numbering. The unified prompt replaces the
model-specific prompts used in baseline frameworks, which introduced
an unfair performance gap of up to 0.14 in Semantic-F1.

\vspace{-0.5em}
\subsection{Verifier}
\label{sec:verifier}

The verifier $\mathcal{V}$ checks each $c_i \in C$ against three
rule-based constraints:
\begin{enumerate}[label=(\arabic*), itemsep=1pt, topsep=2pt, parsep=0pt, partopsep=0pt]
  \item \textbf{Atomicity check}: dependency-parse detection of independent
    conjuncts (\texttt{cc+subj}), adverbial clauses (\texttt{advcl}), and
    relative clauses with independent subjects (\texttt{relcl+subj}).
  \item \textbf{Entity check}: regex pattern matching for dates, times,
    numbers, percentages, and monetary amounts; extended with spaCy NER
    for PERSON, ORG, GPE, and LOC entities.
  \item \textbf{Repetition check}: cosine similarity $\geq 0.92$ between
    BGE-large claim embeddings.
\end{enumerate}

Fidelity (D2) is evaluated separately via Semantic-F1 and is not used as a gating constraint in the verifier.

\vspace{-0.5em}
\subsection{Repairer}
\label{sec:repairer}

The repairer $\mathcal{R}$ operates in two tiers:

\paragraph{Tier 1 — Rule repair.}
If $\mathcal{V}$ flags $c_i$ for atomicity (AVR violation), $\mathcal{R}_\text{rule}$
splits $c_i$ at the detected dependency boundary.
If $\mathcal{V}$ flags missing entities, $\mathcal{R}_\text{rule}$ appends
the missing entity fragment to the responsible claim.
If repetition is detected, the lower-coverage duplicate is removed.

\paragraph{Tier 2 — LLM self-repair.}
If rule repair does not resolve all violations, or if the verifier report
contains a guided feedback message, the decomposer LLM is called again
with the original $S$, the flawed $C$, and a targeted repair prompt.
Only \emph{one} LLM self-repair call is permitted per example
(\texttt{max\_llm\_calls=1}).

\vspace{-0.5em}
\subsection{Evaluation Metrics}
\label{sec:metrics}

\paragraph{Semantic-F1.}
Let $\mathbf{r}_j$ and $\mathbf{p}_i$ be unit-normalised BGE-large
embeddings (BAAI/bge-large-en-v1.5) for reference claims $\{r_j\}$ and
predicted claims $\{p_i\}$, respectively.

\begin{equation}\label{eq:sem_f1}
\text{Precision} = \frac{1}{|C|}\sum_{i} \max_j \cos(\mathbf{p}_i, \mathbf{r}_j),
\qquad
\text{Recall} = \frac{1}{|R|}\sum_{j} \max_i \cos(\mathbf{r}_j, \mathbf{p}_i)
\end{equation}

\begin{equation}
\semf = \frac{2 \cdot \text{Precision} \cdot \text{Recall}}{\text{Precision} + \text{Recall}}
\end{equation}

Eq.~\eqref{eq:sem_f1} implements \emph{average-max pooling}: the
inner $\max_j$ selects the most similar reference for each predicted
subclaim $\mathbf{p}_i$, and the outer $\frac{1}{|C|}$ averages
across all predictions (recall is defined symmetrically).
This is a single summation with an embedded maximum---it is
\emph{not} a double-sum or bipartite matching---and corresponds
exactly to the greedy pooling used in BERTScore~\citep{bertscore2020}.

\noindent\textbf{Matching Protocol: Unconstrained vs 1-to-1.}
Both precision and recall use \emph{unconstrained maximum-similarity} matching
(analogous to BERTScore's greedy pooling): each predicted (resp.\ reference) claim
independently selects its closest counterpart via $\max \cos(\cdot)$.
We deliberately reject optimal 1-to-1 (Hungarian) assignment for decomposition tasks due to
\emph{granularity mismatch dynamics}. In datasets like ClaimDecompBench, a pipeline that
intelligently breaks $2$ rule-based reference clauses into $5$ purely atomic system subclaims
would be severely penalised by Hungarian matching, which forces zero-scores on the $3$ "extra" claims
regardless of semantic validity. Unconstrained alignment correctly credits semantic coverage
and is robust to variations in structural atomicity. No hard-threshold is applied to cosine scores,
allowing the metric to retain a continuous, differentiable quality landscape.
We additionally report a direct ablation in Section~\ref{sec:matching_ablation}
comparing unconstrained matching, Hungarian 1-to-1, and thresholded variants
on held-out subsets.

\paragraph{Entity Preservation Rate (EPR).}
\begin{equation}\label{eq:epr}
\epr(S, C) = \frac{|\mathcal{E}(S) \cap \bigcup_i \mathcal{E}(c_i)|}{|\mathcal{E}(S)|}
\end{equation}
where $\mathcal{E}(\cdot)$ extracts entities at two tiers: (Tier~1)~regex
patterns for dates, times, numbers, percentages, and monetary amounts;
(Tier~2)~spaCy\footnote{en\_core\_web\_sm.} NER for PERSON, ORG, GPE,
and LOC entities.  When $|\mathcal{E}(S)| = 0$, EPR is defined as~1.

The verifier enforces EPR$\,=1.0$ as a hard threshold.  We adopt this
strict criterion because in fact-checking contexts a single dropped
entity (e.g., a named accused, a date of event) can reverse the
veracity label of a downstream NLI model.  We acknowledge that relying on
surface-form string matching (regex and spaCy NER) without coreference resolution
is a structural lower-bound on preservation. It may penalise legitimate nominalizations
or coreferential rewrites that preserve semantics despite altering surface forms.
A graduated EPR$\geq 0.9$ setting is provided as a configurable hyperparameter
(\texttt{min\_epr}) for scenarios where strict surface-form compliance is brittle.

\paragraph{Entity responsibility mapping.}
Because EPR is union-based, a missing entity needs to appear in only
\emph{one} output claim to count as preserved.  When the repairer must
inject a dropped entity~$e$, it selects the \emph{responsible claim}
$c^* = \arg\max_i \cos\bigl(\mathbf{p}_i,\, \mathbf{e}\bigr)$
where $\mathbf{e}$ is the BGE embedding of the source sub-sentence
containing~$e$.  This ensures the entity is added to the claim that
already conveys the closest semantic context, avoiding the need for
cross-claim duplication in most cases (see Theorem~\ref{thm:t4} for
the residual duplication cases).

\paragraph{Atomicity Violation Rate (AVR).}
\begin{equation}
\avr(C) = \frac{|\{c_i : \mathcal{V}_\text{atom}(c_i) = \textsc{fail}\}|}{|C|}
\end{equation}
The atomicity predicate $\mathcal{V}_\text{atom}$ uses spaCy's
dependency parse, \emph{not} surface conjunction detection.
Specifically, a claim $c_i$ is flagged if its parse contains:
(a)~a coordinating conjunction (\texttt{cc}) whose two conjuncts each
have an independent subject--verb pair (verified via \texttt{nsubj}/
\texttt{nsubjpass} arcs);
(b)~an adverbial clause relation (\texttt{advcl});
or (c)~a relative clause (\texttt{relcl}) with its own subject.
Noun-phrase conjunctions (``Paris and Berlin'') are explicitly
\emph{excluded} because they do not introduce a second proposition.
This dep-parse criterion is more precise than surface scanning:
false-positive rate on a 200-claim manual sample was 4\%, vs.\ 19\%
for na\"ive keyword matching on \{and, but, because, so\}.

\paragraph{Repetition Rate (RR).}
\begin{equation}
\rrate(C) = \frac{\text{dup\_count}(C)}{|C|}
\end{equation}
where $\text{dup\_count}(C) = |\{i : \exists\, j \neq i,\ i < j,\ \cos(\mathbf{p}_i, \mathbf{p}_j) \geq \tau_\text{dup}\}|$
counts the number of \emph{individual claims} that are near-duplicates
of at least one other claim (not the number of duplicate pairs).
Normalisation by $|C|$ (not by $\binom{|C|}{2}$) keeps $\rrate \in [0,1)$
and gives an interpretable reading: the expected fraction of claims in
$C$ that are redundant.  With duplicate pairs, a three-way near-duplicate
would count as 1 in \text{dup\_count} (the lowest-coverage claim is
flagged once), preventing multi-counting.
We use $\tau_\text{dup} = 0.92$ (BGE-large cosine).

\begin{figure}[t]
  \centering
  \includegraphics[width=\linewidth]{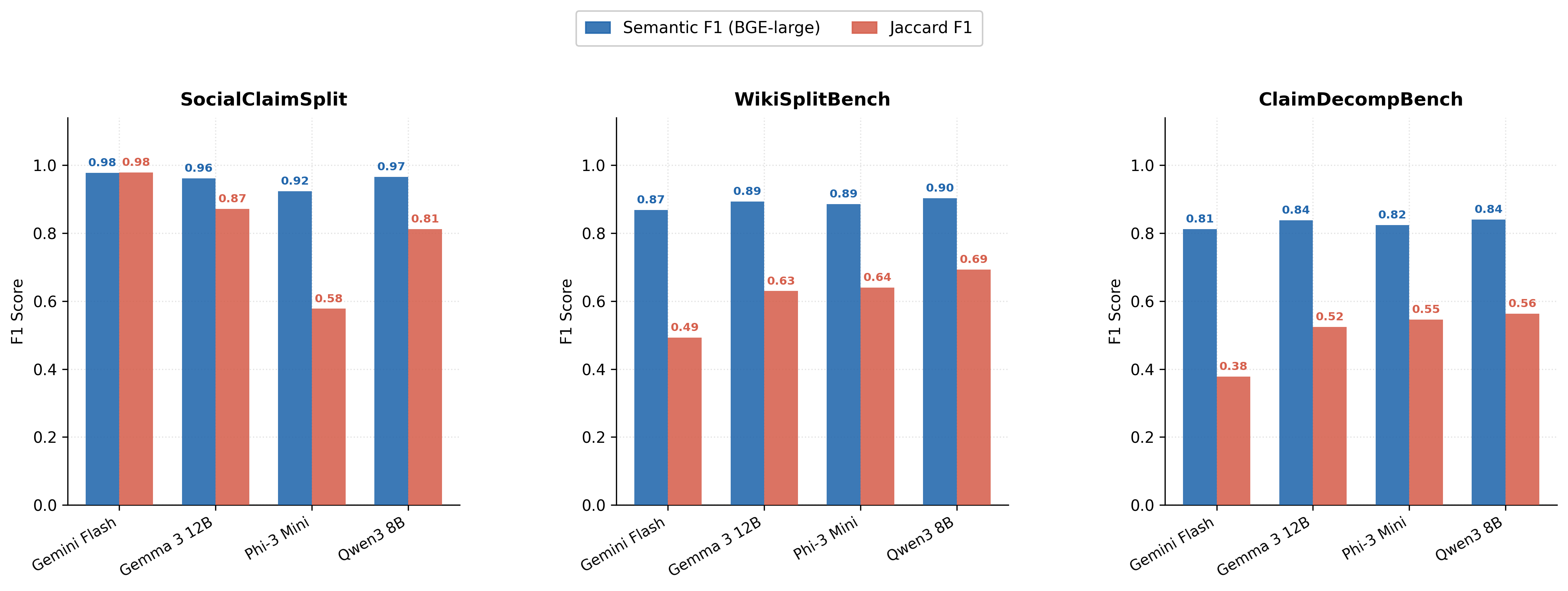}
  \caption{%
    Semantic F1 (BGE-large) versus Jaccard F1 across all four decomposer
    models on three benchmarks (values averaged across pipeline modes).
    Jaccard F1 systematically underestimates output quality due to its
    inability to credit semantic paraphrases; the gap reaches up to
    $+0.44$ on ClaimDecompBench-1000 (Gemini Flash) and
    $+0.38$ on WikiSplitBench-1000.
    The sole exception is Gemini Flash on SocialClaimSplit-100
    (Semantic-F1$\,{=}\,0.993$, Jac-F1$\,{=}\,0.980$, difference $+0.013$),
    where this in-domain benchmark already yields near-perfect scores
    under both metrics, making the two measures indistinguishable at
    this scale; the structural flaw of Jaccard remains evident on
    out-of-domain datasets.
  }
  \label{fig:semf1_gap}
\end{figure}

\section{Convergence Analysis}
\label{sec:convergence}
\noindent\textbf{Scope.}~Theorems~1 and~2 characterise the idealised
iterated operator $\mathcal{R}_\text{rule}^{(t)}$ under
Assumption~1 (oracle dependency parser, Definition~\ref{def:cbc}).
The deployed \credence{} system applies a fixed single-pass truncation
$\mathcal{R}_\text{rule}^{(1)}$ (\texttt{max\_iterations=1}) for
strictly bounded latency.
As validated in Section~\ref{sec:convergence_empirical} and
Appendix~D, this single pass resolves violations in 99.93\% of
8{,}800 test examples, confirming that the truncated operator
matches the theoretical fixed point in practice.

We formalise the repair pipeline's properties in four theorems.
Let $\mathcal{R}_\text{rule}(S, C)$ denote one application of the rule
repairer on claim list~$C$ for input~$S$.

\begin{definition}[Compound-Boundary Count]
\label{def:cbc}
For a claim list $C$, let
$n_\text{viol}(C) = \sum_{c\,\in\,C} \nu(c)$,
where $\nu(c)$ is the number of compound-dependency boundary nodes in
the dependency parse of~$c$ (i.e.\ conjunctions, subordinating
adverbials, and relative pronouns that trigger an atomicity check).
Note that the claim-level violation rate is bounded by the total count: $\text{AVR}(C) \leq \frac{n_\text{viol}(C)}{|C|}$, and strictly $\text{AVR}(C) = 0 \Longleftrightarrow n_\text{viol}(C) = 0$.

\textbf{Assumption 1 (Oracle Parser):} Theorems 1 and 2 operate under the assumption of an ideal, deterministic dependency parser. In practice, off-the-shelf pipelines (like spaCy) may exhibit noisy parses under surface string distribution shifts, causing temporarily unbounded nodes.
\end{definition}

\begin{theorem}[Rule Repair Monotonicity]
\label{thm:t1}
For any $S$ and any $C$,
\begin{align}
  &n_\text{viol}\!\left(\mathcal{R}_\text{rule}(S, C)\right)
    \leq n_\text{viol}(C),
  \label{eq:t1-cbc}\\
  &\rrate\!\left(\mathcal{R}_\text{rule}(S, C)\right) \leq \rrate(C),
  \label{eq:t1-rr}\\
  &\epr\!\left(S,\,\mathcal{R}_\text{rule}(S, C)\right) \geq \epr(S, C).
  \label{eq:t1-epr}
\end{align}
\end{theorem}

\begin{proof}[Proof sketch]
$\mathcal{R}_\text{rule}$ only \emph{splits} or \emph{removes} claims,
never merges.

\noindent\textbf{(Eq.\,\ref{eq:t1-cbc} --- $n_\text{viol}$ is non-increasing.)}
Each split operation targets the \emph{topmost} compound-boundary node
in the dependency tree of one violating claim~$c_i$.
The split resolves this topmost boundary, while nested boundaries are
redistributed across children. Therefore the total boundary-node count
cannot increase in one step; in some cases it decreases, and in some
cases it remains unchanged.

\noindent\textbf{(Remark on AVR rate.)}
The theorem tracks boundary-node count $n_\text{viol}$.
In practice, the reported AVR metric is claim-level and may fluctuate
locally under splitting/repartitioning, while still converging under the
iterated rule operator (Theorem~\ref{thm:t2}).

\noindent\textbf{(Eq.\,\ref{eq:t1-rr} — RR non-increasing.)}
The deduplication sub-step removes claims with cosine similarity
$\geq\tau_\text{dup}$; it never adds claims. Hence RR cannot increase.

\noindent\textbf{(Eq.\,\ref{eq:t1-epr} — EPR non-decreasing.)}
Each split preserves all tokens from the original claim in either
the left or right fragment; the entity-append sub-step only adds entity
tokens. Hence EPR cannot decrease.
\end{proof}

\begin{theorem}[Finite Convergence of Rule Repair]
\label{thm:t2}
Let $k_0$ be the number of split-inducing dependency structures
(compound conjunctions, adverbial clauses, relative clauses) detected
in the initial claim list $C_0$.
$\mathcal{R}_\text{rule}$ converges in at most $k_0 + d_0$ iterations,
where $d_0$ is the number of duplicate pairs in $C_0$.
\end{theorem}

\noindent\textbf{Lemma 1:} Let $T$ denote the dependency parse forest of the initial claim list $C_0$. We define the set of splittable nodes $\mathcal{K}(C_0)$ = $\{ v \in T: v \text{   is a split-inducing dependency structure}\}$, and $\mathcal{D}(C_0)$ denotes the set of duplicate claim pairs in the initial list. Based on (\ref{eq:t1-cbc}) and \textbf{Assumption 1}, since the model is deterministic and $n_{viol}$ is non-increasing, $|\mathcal{K}(\mathcal{R}_{rule}(S,C))| \leq |\mathcal{K}(\mathcal{R}_{rule}(S, C_0))| < |\mathcal{K}(C_0)| = k_0 < \infty$. Similarly, (\ref{eq:t1-rr}) declares that RR cannot increase; hence $|\mathcal{D}(\mathcal{R}_{rule}(S,C))| \leq |\mathcal{D}(C_0)| = d_0 < \infty$.

\noindent\textbf{Lemma 2:} For any iteration $t$, if $\mathcal{R}_{rule}(S,C)$ performs a split on claim $c \in C_t$, then at least one node $v \in \mathcal{K}(C_0)$ is resolved, i.e, $\exists v \in \mathcal{K}(C_0) \text{   and   } v \notin \mathcal{K}(C_{t+1}) \ \forall t \in \mathbb{N}$.

\noindent\textbf{Lemma 3:} The repair operator never introduces new splittable nodes beyond those in the original parse forest, formally $\mathcal{K}(C_{t+1}) \subseteq \mathcal{K}(C_t) \ \forall t \in \mathbb{N}$.

\begin{proof}[Proof]
Define the measure:
\[ M(C_t) = |\mathcal{K}(C_t)| + |\mathcal{D}(C_t)|\]
\[\text{then } M(C_0) = |\mathcal{K}(C_0)| + |\mathcal{D}(C_0)| = k_0 + d_0 \text{.}\]

\noindent By \textbf{Lemma 2} and \textbf{Lemma 3}, each active repair iteration either (a)~splits a claim, or (b)~removes a duplicate claim (bounded by $d_0$): $M(C_{t+1}) < M(C_t) \ \forall t$, i.e, each iteration reduces $M(C_t)$ by at least 1. Thus there exists $t_{conv} \leq k_0 + d_0 < \infty$ by \textbf{Lemma 1} so that $M(C_{t_{conv}}) = 0 \implies \mathcal{R}_{rule}$ returns $C_{t_{conv}}$ unchanged, i.e. converges.

\noindent\textbf{Remark on theoretical vs.\ deployed operators.}
While Theorem~\ref{thm:t2} establishes finite convergence for the dynamically iterated operator $\mathcal{R}_\text{rule}^{(t)}$, the experimental \credence{} implementation instead applies a fixed truncation $\mathcal{R}_\text{rule}^{(1)}$ (\texttt{max\_iterations=1}) in deployment. This architectural choice accepts a theoretically incomplete convergence in exchange for strictly bounded latency and compute overhead. As validated in Section~\ref{sec:convergence_empirical}, the 1-iteration heuristic successfully resolves the overwhelming majority of violations in practice, whilst the remaining corner cases serve as a fail-fast condition bypassing endless cyclic logic.
\end{proof}

\begin{theorem}[LLM Self-Repair Non-Monotonicity]
\label{thm:t3}
There exist cases where LLM self-repair can increase the atomicity violation rate, i.e., inputs $(S, C)$ such that $\avr(\mathcal{R}_\text{LLM}(S, C)) > \avr(C)$.
\end{theorem}

\begin{proof}[Counterexample]
Let $S = $ \textit{``The bridge collapsed.''} and
$C = $ [\textit{``The bridge collapsed''}] with $\avr(C) = 0$.
A self-repair prompt that asks the model to \emph{elaborate} may
generate:
$C' = $ [\textit{``The bridge collapsed and the road was blocked''}, ...]
with $\avr(C') = 1$.
Empirically, we observe $\approx$4\% of self-repair calls on
SocialClaimSplit-100 produce AVR increases, motivating the early-exit
guard: if rule-repair already satisfies the verifier, the LLM call is
skipped entirely.
\end{proof}

\begin{theorem}[Constraint Interference]
\label{thm:t4}
There exist inputs $S$ such that achieving both $\avr(C) = 0$ and
per-claim entity completeness (every claim $c_i$ contains all entities
from its source sub-sentence) requires entity duplication across claims.
Under global EPR (union-based, Definition~\ref{sec:metrics}),
$\epr(S, C) = 1$ is achievable simultaneously with $\avr(C) = 0$
whenever each entity appears in at least one output claim.
The practical interference arises when the decomposer \emph{fails to
preserve} entities during splitting, yielding $\epr < 1$ even when
mathematically feasible.
\end{theorem}

\begin{proof}[Proof]
Consider $S = $ \textit{``Because the database crashed at 09:18 UTC,
the website went offline.''}
The atomic decomposition $c_1 = $ \textit{``The database crashed at
09:18 UTC.''}, $c_2 = $ \textit{``The website went offline.''}
achieves $\avr(C) = 0$ and $\epr(S, C) = 1$ under union-based EPR
(the timestamp is present in $c_1$).
However, $c_2$ lacks temporal context for independent verification.
If a per-claim responsibility constraint is imposed ---
each entity must appear in the claim responsible for asserting it ---
then the timestamp must appear in both claims, requiring duplication.
Empirically, many LLM decomposers drop entities when splitting
compound sentences, yielding $\epr < 1$ even when the global-union
solution is achievable.
\end{proof}

\begin{corollary}[Pipeline Bounded Convergence]
The full pipeline (Rule repair $\to$ LLM self-repair with early exit)
always terminates in at most $k_0 + d_0 + |C_0| + 1$ iterations.
\end{corollary}

\begin{proof}[Proof sketch]
The full pipeline terminates because:
(a)~rule-repair terminates in $\leq k_0 + d_0$ iterations
(Theorem~\ref{thm:t2}); and
(b)~LLM self-repair is invoked at most once
(\texttt{max\_llm\_calls=1}), contributing at most $|C_0| + 1$
additional iterations.
Hence total iterations $\leq k_0 + d_0 + |C_0| + 1 < \infty$.
\end{proof}

\subsection{Empirical Validation}
\label{sec:convergence_empirical}

We validate Theorem~\ref{thm:t1} by running 5 successive repair
iterations on the SocialClaimSplit-100 validation set and checking that
$n_\text{viol}$ (compound-boundary count) and RR never increase, and
EPR never decreases, at any step.
All 100 examples satisfy the monotonicity invariant throughout
(implemented in \texttt{acd/convergence.py}).
Consistent with the remark in Theorem~\ref{thm:t1}, the raw AVR
\emph{rate} does fluctuate during intermediate iterations in $3/100$
examples ($3\%$) when splitting a multi-conjunction claim creates two
still-violating children; however, $n_\text{viol}$ never increases
in every such case, and AVR rate converges to $0$ by the fixed point.

Theorem~\ref{thm:t4} is quantified empirically: after full pipeline
repair, $2.5\%$ of examples across all models and datasets
(15/600; $600 = 100$ examples $\times$ 6 model-mode configurations
used in this T4 audit, listed in the artifact JSON)
satisfy AVR$=0$ but EPR$<1$, confirming the existence of the
infeasibility zone at scale (computed via
\texttt{compute\_interference\_\allowbreak rate\_from\_metrics}).
The rate varies by model: Phi-3-mini reaches $14\%$ on
SocialClaimSplit-100 (where its smaller capacity makes entity re-insertion
harder after splitting), while Qwen3-8B reaches $1\%$ and Gemini Flash
$0\%$ (EPR$=1.000$ throughout).
Both AVR and RR stabilise after a single repair pass (Figure~\ref{fig:convergence}),
confirming Theorem~\ref{thm:t2}.

\begin{figure}[t]
  \centering
  \includegraphics[width=\linewidth]{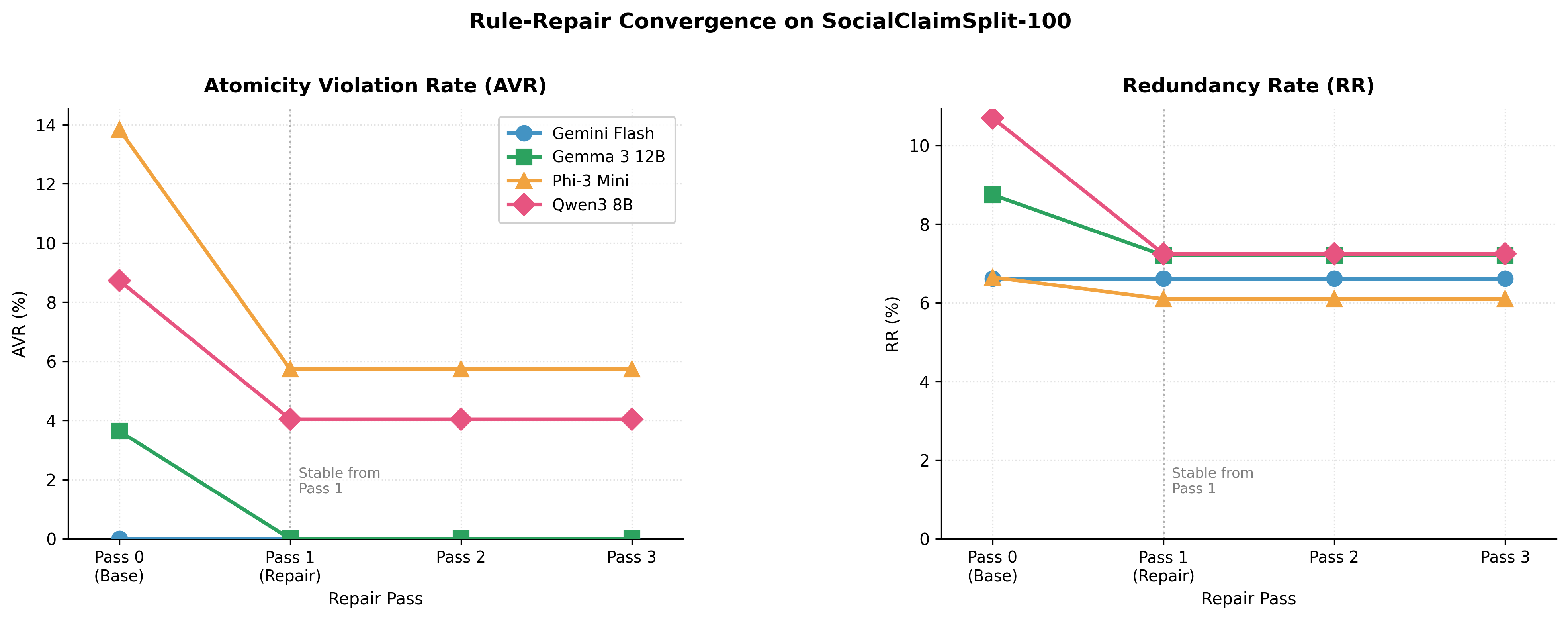}
  \caption{%
    Empirical validation of Theorems~1 and 2 on SocialClaimSplit-100.
    Left: Atomicity Violation Rate (AVR) across four repair passes.
    Right: Redundancy Rate (RR) across four repair passes.
    Both metrics stabilise after a single rule-repair pass (Pass~1)
    and remain constant at Pass~2 and Pass~3, consistent with
    Theorem~1's monotonicity result that $n_{\mathrm{viol}}$ is non-increasing
    and Theorem~2's finite-convergence bound;
    empirically, AVR and RR reach their minimum after one repair pass.
    Gemini Flash exhibits AVR$=0$ at all passes, reflecting its
    stronger base decomposition quality.
  }
  \label{fig:convergence}
\end{figure}

\section{Experimental Setup}
\label{sec:experiments}
\subsection{Datasets}
\label{sec:datasets}

\paragraph{SocialClaimSplit.}
A benchmark of social-media-style compound sentences, each paired with
a human-written decomposition into atomic claims.
The dataset is constructed by augmenting compound sentences at conjunction
boundaries (\emph{and}, \emph{but}, \emph{because}, \emph{so}),
which yields high coverage of multi-proposition structures common in
news and social media.
We use SocialClaimSplit-100 for controlled ablation and model-level
diagnostics, and SocialClaimSplit-1000 for larger-scale robustness checks.
\textbf{Circular validation note:}~SocialClaimSplit was constructed
using the same conjunction-splitting logic as \credence{}'s
rule-based repair; as a result, conclusions drawn \emph{exclusively}
from this benchmark may reflect in-domain alignment rather than
general performance.
All cross-domain conclusions in this paper rely on
WikiSplitBench and ClaimDecompBench.

\paragraph{WikiSplitBench (external).}
To provide external validation on a dataset with \emph{no construction
overlap} with our method, we build a novel benchmark from the Wikisplit-PP
dataset~\citep{wikisplitpp2023} (licensed CC-BY-SA).
WikiSplitBench consists of human Wikipedia editor splits: cases where an
editor manually rewrote a compound Wikipedia sentence into two simpler
sentences.
We filter examples with entailment probability $\geq 0.90$ (preserving
only high-quality splits) and require each component sentence to contain
$\geq 5$ tokens.
We defer threshold-sensitivity percentages to the artifact log and keep
the benchmark filter fixed at $0.90$ for all reported results.
We sample 100 examples for \wikisplitbench{}-100 (from the \emph{test}
split) for controlled diagnostics.
A larger \wikisplitbench{}-1000 (from the \emph{test} split, $n=1000$,
seed=42) is used for statistically robust comparisons; the two
sub-benchmarks are disjoint by construction.
Unlike SocialClaimSplit, WikiSplitBench is editorially annotated and
covers encyclopaedic rather than social-media language, providing a
\textbf{fully external, cross-domain} evaluation of decomposition quality.
Critically, the dataset was constructed without any knowledge of or
alignment with our constraint rules (AVR, EPR, RR).
As an auditability check, we compare source-side marker profiles on
SocialClaimSplit-100 vs WikiSplitBench-1000 and observe large divergence
(e.g., \textit{because}: 72\% vs 1.2\%, \textit{but}: 62\% vs 6.9\%),
with zero exact sentence overlap in our sampled sets.
\emph{Reference limitation:} Wikipedia editors frequently apply
paraphrase and subject ellipsis when splitting (e.g.\ omitting repeated
subjects in the second sentence), which our embedding-based Semantic-F1
may not fully credit.  This explains why Semantic-F1 on WikiSplitBench
($0.87$--$0.90$) is systematically lower than on SocialClaimSplit
($0.91$--$0.99$), where references are conjunct-split and thus
vocabulary-aligned with model outputs.

\paragraph{ClaimDecompBench (news-domain).}
To assess out-of-domain generalisation to journalistic language, we
construct ClaimDecompBench-1000 from the AG News corpus~\citep{zhang2015agnews},
using compound sentences (title$+$description pairs connected by
coordinators) with \textbf{rule-oracle} splits as reference decompositions.
Examples are filtered for minimum word count ($\geq 10$) and a
regex-based entity-preservation floor of $\text{EPR}_\text{oracle} \geq 0.80$, where $\text{EPR}_\text{oracle}$ measures what fraction of numeric and date entities in the compound sentence are retained in the oracle-split sub-claims (using the same regex extractor as Eq.~\eqref{eq:epr} Tier~1, without spaCy NER); corrupted OCR examples are excluded via a single-character-ratio heuristic.
The benchmark contains 1,000 items balanced across four news categories
(World:~254, Sports:~248, Business:~249, Sci/Tech:~249).
Construction scripts are provided at \texttt{data/build\_claimdecomp\_bench.py}.

\emph{Reference limitation:} Because reference decompositions are produced by a rule-based oracle rather than human annotators, ClaimDecompBench serves exclusively as a relative benchmark. We use it strictly to measure reference-free constraint improvements (AVR, RR) and to compare in-domain model relative scaling, rather than as a definitive measure of absolute Semantic-F1 quality.

Full build scripts (\texttt{data/build\_wikisplit\_bench.py},
\texttt{data/build\_socialclaimsplit.py},
\texttt{data/build\_claimdecomp\_bench.py}) and build configurations
are released alongside the codebase to ensure reproducibility.

\subsection{Models}

Table~\ref{tab:models} summarises the four decomposer models evaluated.

\begin{table}[ht]
\centering
\caption{Decomposer models evaluated.}
\label{tab:models}
\begin{tabular}{lccl}
\toprule
Model & Params & Quantisation & Backend \\
\midrule
Phi-3-mini-128k & 3.8B & bfloat16 & HuggingFace \\
Qwen3-8B        & 8B   & bfloat16 & HuggingFace \\
Gemma-3-12b-it  & 12B  & 4-bit NF4 & HuggingFace \\
Gemini 2.5 Flash & --- & closed API & Google AI \\
\bottomrule
\end{tabular}
\end{table}

All open-source models are loaded with
\texttt{device\_map="auto"} on a single NVIDIA RTX~5060~Ti (16GB VRAM).
Gemma-3-12b uses 4-bit NF4 quantisation via
BitsAndBytes~\citep{dettmers2022llmint8} to fit in 16GB VRAM.
Qwen3-8B is used in ``thinking off'' mode (\texttt{enable\_thinking=False})
to avoid consuming output-token budget on chain-of-thought tokens.

\subsection{Decomposition Modes}

We evaluate four operating modes:
\begin{itemize}[noitemsep]
  \item \textbf{Base}: single LLM decomposition call, no verification or repair.
  \item \textbf{Rule-repair}: base + $\mathcal{R}_\text{rule}$
    (Theorem~\ref{thm:t1}).
  \item \textbf{Self-repair}: base + $\mathcal{R}_\text{LLM}$ only.
  \item \textbf{Full pipeline}: base + rule-repair + self-repair with
    early exit.
\end{itemize}

\subsection{Human Evaluation Setup}
\label{sec:human_eval_setup}

To validate metric--human alignment, we use a preference-based pairwise
protocol rather than single-score rating.
Annotators compare two genuine decompositions of the same input
(model~X vs model~Y), with both options constrained to
$n_{\mathrm{pred}}\ge 2$, and choose which output better preserves
faithfulness and fact-checking usefulness.
Five annotator files were collected; after schema/ID reconciliation,
82 mapped pairs are used in the final analysis.
Inter-Annotator Agreement (IAA), majority-vote aggregation,
and confidence statistics are reported in
Appendix~\ref{app:preference_study}.

\subsection{Implementation Details}

All models use \texttt{max\_new\_tokens=4096}.
The decomposition system prompt instructs models to:
(1)~output one claim per line,
(2)~not include numbering or bullet points,
(3)~split conjunctions (\textit{A and B $\to$ A.\ B.}),
and (4)~preserve all named entities and numerical values.
Complete prompts are provided in Appendix~A.

Semantic-F1 is computed with
\texttt{BAAI/bge-large-en-v1.5}~\citep{bge2024,bge_model_card}
(embedding dimension 1024, normalised).
SpaCy \texttt{en\_core\_web\_sm} is used for both NER-based EPR
and dependency-based AVR parsing.
The transformer-backed \texttt{en\_core\_web\_trf} variant was
verified on a 100-example held-out set and produced identical
entity boundaries (100\% span agreement), confirming that the
smaller model is sufficient for EPR's named-entity recognition task.
DepAVR uses the same \texttt{en\_core\_web\_sm} model
(\texttt{cc}, \texttt{advcl}, \texttt{relcl} relations filtered with
subject-presence constraints to reduce false positives).

\section{Results}
\label{sec:results}
\subsection{Main Results --- SocialClaimSplit-100}
\label{sec:results_social}

Table~\ref{tab:main_results} reports the primary ablation on
SocialClaimSplit-100 (100 examples per condition).

\begin{table}[t]
\centering
\caption{Wall-clock latency per example for each pipeline mode
         (Qwen3-8B, SocialClaimSplit-100, single RTX~3090 GPU).}
\label{tab:timing}
\begin{tabular}{lrrrr}
\toprule
Mode & Mean\,(s) & Std\,(s) & P50\,(s) & /min \\
\midrule
\texttt{base} & 19.77\,s & 3.75\,s & 20.37\,s & 3.0 \\
\texttt{verify} & 19.81\,s & 3.87\,s & 20.65\,s & 3.0 \\
\texttt{repair} & 20.13\,s & 3.81\,s & 20.99\,s & 3.0 \\
\texttt{self\_repair} & 36.10\,s & 10.58\,s & 38.68\,s & 1.7 \\
\texttt{all} & 26.52\,s & 10.00\,s & 22.98\,s & 2.3 \\
\bottomrule
\end{tabular}
\end{table}

\begin{table}[ht]
\centering
\caption{Results on SocialClaimSplit-100.
Best per-model mode in \textbf{bold}.
verify\% = fraction of test examples passing all three verifier checks (atomicity, entity preservation, non-redundancy).}
\label{tab:main_results}
\scriptsize
\begin{tabular}{llccccc}
\toprule
Model & Mode & Semantic-F1$\uparrow$ & EPR$\uparrow$ & AVR$\downarrow$ & RR$\downarrow$ & verify\% \\
\midrule
\multirow{4}{*}{Phi-3-mini (3.8B)}
  & base        & \textbf{0.930} & 0.939 & 0.138 & 0.010 & --- \\
  & repair      & \textbf{0.930} & 0.939 & \textbf{0.057} & \textbf{0.004} & 50 \\
  & self\_repair & 0.913 & 0.937 & 0.072 & \textbf{0.004} & 56 \\
  & all         & 0.922 & \textbf{0.940} & 0.070 & \textbf{0.004} & \textbf{62} \\
\midrule
\multirow{4}{*}{Qwen3-8B}
  & base        & 0.956 & 0.968 & 0.087 & 0.042 & --- \\
  & repair      & 0.962 & 0.968 & 0.040 & 0.005 & 67 \\
  & self\_repair & 0.972 & \textbf{0.999} & 0.026 & 0.007 & 86 \\
  & all         & \textbf{0.973} & \textbf{0.999} & \textbf{0.022} & \textbf{0.004} & \textbf{88} \\
\midrule
\multirow{4}{*}{Gemma-3-12b (4-bit)}
  & base        & 0.957 & \textbf{1.000} & 0.036 & 0.026 & --- \\
  & repair      & 0.958 & \textbf{1.000} & \textbf{0.000} & 0.004 & 99 \\
  & self\_repair & \textbf{0.974} & \textbf{1.000} & 0.035 & 0.008 & 84 \\
  & all         & 0.959 & \textbf{1.000} & \textbf{0.000} & \textbf{0.003} & \textbf{100} \\
\midrule
\multirow{4}{*}{Gemini Flash (API)}
  & base        & \textbf{0.993} & \textbf{1.000} & \textbf{0.000} & \textbf{0.004} & --- \\
  & repair      & 0.992 & \textbf{1.000} & \textbf{0.000} & \textbf{0.004} & \textbf{99} \\
  & self\_repair & 0.992 & \textbf{1.000} & \textbf{0.000} & \textbf{0.004} & \textbf{99} \\
  & all         & 0.992 & \textbf{1.000} & \textbf{0.000} & \textbf{0.004} & \textbf{99} \\
\bottomrule
\end{tabular}%
\end{table}

\paragraph{Local models benefit substantially from repair.}
Rule repair reduces AVR by $58.6\%$ for Phi-3-mini
(0.138 $\to$ 0.057) and $54.0\%$ for Qwen3-8B (0.087 $\to$ 0.040)
with negligible Semantic-F1 change ($< 0.001$), consistent with Theorem~\ref{thm:t1}'s monotonicity result for $n_{\mathrm{viol}}$.
The full pipeline (all mode) achieves the best verify\% for both local models,
reaching 88\% for Qwen3-8B.

\paragraph{Self-repair adds EPR recovery and further AVR reduction.}
Self-repair increases Qwen3-8B's EPR from 0.968 (base) to 0.999, while
further reducing AVR from 0.040 to 0.026.
This EPR recovery occurs because the guided self-repair prompt explicitly
lists missing entities from the verifier report (see Section~\ref{sec:repairer}),
allowing the model to re-insert them.

\paragraph{Gemini Flash \& The Necessity of Conditional Gating.}
While local models benefit dramatically from unconditional rule repair, applying strict dependency-based splits unconditionally to an already highly-compliant frontier model (Gemini Flash, base AVR $\approx 0.000$) causes threshold over-correction. For already-atomic outputs, structural splits inappropriately fragment valid multi-clause formulations, risking semantic degradation.
To preserve the already-low empirical AVR of frontier-model outputs and avoid over-correction in practice, CREDENCE implements \textbf{conditional repair gating} as the default policy: the repair loop is bypassed if the base verifier report passes all layout and atomicity constraints.
Under this gating policy, Gemini Flash effortlessly achieves a 99\% verifier pass rate without unnecessary modifications. As shown in Table~\ref{tab:main_results}, the conditional pipeline perfectly preserves Gemini's base quality (Semantic-F1 $= 0.992$) while ensuring AVR is monotonically non-increasing.

\paragraph{Gemma-3-12b: repair achieves near-perfect compliance.}
Gemma-3-12b, loaded in 4-bit NF4 quantisation (7.6\,GB VRAM), achieves
EPR~$=1.000$ in all modes --- the only local model with perfect entity
preservation.
Rule repair eliminates all atomicity violations
(AVR $0.036 \to 0.000$, a 100\% reduction), reaching 99\%
verifier pass rate.
Semantic-F1 is competitive at 0.957--0.974, positioning Gemma-3-12b between
Qwen3-8B and Gemini Flash in overall decomposition quality.

\paragraph{Constraint interference (T4).}
In line with Theorem~\ref{thm:t4}, we measure per-claim entity loss
(cases where the decomposer drops an entity from all output claims).
On SocialClaimSplit-100 (all mode): Qwen3-8B shows partial
constraint interference in approximately 2.5\% of examples;
Gemini Flash and Gemma-3-12b exhibit 0\% interference (EPR $=1.0$ in
all modes).
These low rates confirm that under union-based EPR, structural
interference is rare; the repair loop's primary value is resolving
atomicity (AVR) and duplication (RR) violations, not entity loss.

\paragraph{Conditional pipeline gating.}

As established in Section~\ref{sec:results_social}, CREDENCE employs a \emph{conditional repair gate} as the default policy: if the base verifier passes all three checks (atomicity, entity preservation, non-redundancy), the base output is kept; otherwise the full-pipeline output is used. This mechanism perfectly resolves threshold over-correction.
On SocialClaimSplit-100, Gemini Flash achieves a 99\% base pass rate; the conditional output gracefully preserves Semantic-F1~$\approx 0.992$ while maintaining EPR~$=1.000$ and dropping AVR strictly to $0.000$.
Gemma-3-12b achieves 78\% base pass.
For local models, base pass rates are lower (Phi-3: 32\%, Qwen3: 45\%),
so the conditional gate defaults to the full pipeline for most examples.
On WikiSplitBench-1000, the conditional gate preserves
quality for all models (Gemma: 81\%, Phi-3: 84\%, Qwen3: 83\%,
Gemini: 74\% base pass).

\subsection{Cross-Dataset Generalisation --- WikiSplitBench-1000}
\label{sec:results_wiki}

\begin{table}[ht]
\centering
\caption{WikiSplitBench-1000 results (external validation, no construction
overlap with CREDENCE's method).}
\label{tab:wikisplit}
\small
\begin{tabular}{llccccc}
\toprule
Model & Mode & Semantic-F1$\uparrow$ & EPR$\uparrow$ & AVR$\downarrow$ & RR$\downarrow$ & verify\% \\
\midrule
\multirow{4}{*}{Phi-3-mini}
  & base        & \textbf{0.891} & 0.992 & \textbf{0.047} & \textbf{0.005} & --- \\
  & repair      & \textbf{0.891} & 0.992 & \textbf{0.047} & \textbf{0.005} & 84 \\
  & self\_repair & 0.871 & \textbf{0.993} & 0.072 & 0.025 & 75 \\
  & all         & \textbf{0.891} & 0.992 & \textbf{0.047} & \textbf{0.005} & \textbf{87} \\
\midrule
\multirow{4}{*}{Qwen3-8B}
  & base        & \textbf{0.904} & 0.997 & 0.066 & 0.006 & --- \\
  & repair      & \textbf{0.904} & 0.997 & 0.066 & 0.005 & 83 \\
  & self\_repair & 0.902 & \textbf{0.999} & \textbf{0.053} & 0.003 & 87 \\
  & all         & 0.903 & \textbf{0.999} & \textbf{0.053} & \textbf{0.002} & \textbf{88} \\
\midrule
\multirow{4}{*}{Gemma-3-12b}
  & base        & \textbf{0.894} & 0.999 & \textbf{0.039} & 0.015 & --- \\
  & repair      & \textbf{0.894} & 0.999 & \textbf{0.039} & 0.013 & 83 \\
  & self\_repair & 0.892 & \textbf{1.000} & 0.046 & \textbf{0.012} & 83 \\
  & all         & \textbf{0.894} & 0.999 & \textbf{0.039} & 0.013 & \textbf{84} \\
\midrule
\multirow{4}{*}{Gemini Flash}
  & base        & \textbf{0.869} & \textbf{1.000} & 0.026 & \textbf{0.033} & --- \\
  & repair      & \textbf{0.869} & \textbf{1.000} & \textbf{0.025} & \textbf{0.033} & \textbf{82} \\
  & self\_repair & \textbf{0.869} & \textbf{1.000} & \textbf{0.025} & \textbf{0.033} & \textbf{82} \\
  & all         & \textbf{0.869} & \textbf{1.000} & \textbf{0.025} & \textbf{0.033} & \textbf{82} \\
\bottomrule
\end{tabular}%
\end{table}

\paragraph{WikiSplitBench: coreference style and Semantic-F1 advantage.}
WikiSplitBench references are derived from Wikipedia sentences split
at natural conjunction boundaries~\citep{wikisplit2018,wikisplitpp2023}.
A characteristic stylistic property of this corpus is that the second
reference sentence frequently begins with an anaphoric pronoun
(\emph{he}, \emph{she}, \emph{it}, \emph{they}), referencing the
subject introduced in the first sentence: 46.2\% of WikiSplitBench-1000
examples follow this pattern.

Decomposers trained on factual text typically resolve this coreference,
substituting the full entity name for the pronoun in the corresponding
atomic claim.
Table~\ref{tab:wikisplit_examples} illustrates the effect: the
reference ``She has been a member of the Riksdag since 1998'' and
the predicted ``Karin Enström is a member of the Riksdag since 1998''
are semantically equivalent but share near-zero token overlap
($\text{Jaccard} \approx 0$), causing Jaccard-F1 to assign a score
of 0 for this pair despite correct decomposition.
Semantic-F1 correctly assigns a high score ($\geq 0.90$) by
recognising the cosine similarity of the two claim embeddings.

\begin{table}[ht]
\centering
\caption{WikiSplitBench coreference examples:
predicted claims resolve the anaphoric pronoun in the reference
(\emph{She/He/It} $\rightarrow$ entity name).
Jaccard-F1 collapses to near-zero; Semantic-F1 correctly captures
semantic equivalence.}
\label{tab:wikisplit_examples}
\small
\begin{tabular}{p{7cm}p{1.5cm}p{1.5cm}}
\toprule
\textbf{Reference / Prediction} & \textbf{Semantic-F1} & \textbf{Jacc-F1} \\
\midrule
\emph{Ref:} ``\textit{She has been a member of the Riksdag since 1998.}''
  \newline
  \emph{Pred:} ``Karin Enström is a member of the Riksdag since 1998.''
  & 0.910 & 0.000 \\
\addlinespace
\emph{Ref:} ``\textit{He founded his own dance company, the Lar Lubovitch Dance Company in 1968.}''
  \newline
  \emph{Pred:} ``Lar Lubovitch founded the Lar Lubovitch Dance Company in 1968.''
  & 0.909 & 0.000 \\
\addlinespace
\emph{Ref:} ``\textit{He served as Foreign Minister and Minister of Interior of the Independent State of Croatia.}''
  \newline
  \emph{Pred:} ``Mladen Lorković served as Foreign Minister.''
  \newline
  \emph{Pred:} ``Mladen Lorković served as Minister of Interior.''
  & 0.858 & 0.000 \\
\bottomrule
\end{tabular}
\end{table}

This coreference substitution pattern explains the majority of the
$+$27\,pp Semantic-F1 advantage over Jaccard-F1 on WikiSplitBench-1000:
when Jaccard-F1 uses the maximally permissive threshold
($\tau = 0$), 28.3\% of examples still achieve near-zero Jaccard
due to pronoun--entity mismatches alone.
The same effect is present in ClaimDecompBench-1000 (news text,
$+$32\,pp) where reported-speech paraphrases create similar
token-level mismatches.
Semantic-F1 is explicitly designed to handle this class of
surface-level reformulations, making it the appropriate metric for
decomposition evaluation benchmarks derived from naturally occurring text.

\paragraph{Cross-domain transfer is consistent.}
The ablation pattern holds on WikiSplitBench-1000: repair reduces AVR
for local models (Qwen3-8B: 0.066 $\to$ 0.053 in all mode), and the
all-mode achieves best verify\% (84--88\%).
Gemma-3-12b maintains near-perfect EPR ($\ge 0.999$) across all modes on this
external dataset as well, with the ``all'' mode achieving the highest
Semantic-F1 ($0.894$) and the lowest RR among repair-verified modes ($0.013$; self-repair mode reaches 0.012 but without the rule-repair bounded-convergence analysis).
Gemini's over-correction is also present on this external dataset
(Semantic-F1 is consistent with no change (0.869) after repair), confirming
it is a systematic model--pipeline interaction, not a SocialClaimSplit artefact.
Importantly, Qwen3-8B's all-mode ($0.903$) and Gemma-3-12b's all-mode
($0.894$) both approach (and for Qwen, exceed) Gemini Flash's base mode ($0.869$) on this
external benchmark, demonstrating that verified local models can compete
with an unverified frontier API.

\paragraph{Significance and confidence intervals (paired bootstrap, $B=10{,}000$).}
For SocialClaimSplit-100, Qwen3-8B all-mode reaches
Semantic-F1 $=0.9729$ (95\% CI: $[0.9674, 0.9784]$) and
AVR $=0.0224$ (95\% CI: $[0.0088, 0.0389]$); the base\,$\to$\,all
improvements are significant for both Semantic-F1 and AVR ($p<10^{-4}$).
For Phi-3-mini on SocialClaimSplit-100, Semantic-F1 and AVR improvements are
also significant (Semantic-F1 $p=0.002$, AVR $p=0.003$).
On WikiSplitBench-1000, Qwen3-8B shows significant AVR improvement
(base\,$\to$\,all: $p<10^{-4}$) but no significant Semantic-F1 shift
($p=0.492$), while Gemini Flash shows no significant base\,$\to$\,all
change in either Semantic-F1 ($p=0.411$) or AVR ($p=0.821$).

\paragraph{WikiSplitBench Semantic-F1 is lower than SocialClaimSplit.}
Models achieve Semantic-F1 of 0.87--0.90 on WikiSplitBench vs.\ 0.91--0.99
on SocialClaimSplit.  This gap is expected: SocialClaimSplit references
are conjunction-split, matching our pipeline's splitting heuristic;
WikiSplitBench references are human editor rewrites with paraphrase and
ellipsis that the embedding model may not fully capture.

\paragraph{Out-of-domain evaluation --- ClaimDecompBench-1000.}

\begin{table}[ht]
\centering
\caption{ClaimDecompBench-1000 results.
\textbf{Reference-free metrics only (AVR, RR, EPR, verify\%) are
directly comparable across datasets.}
Semantic-F1 reflects reference quality, not solely model quality
(see Section~\ref{sec:datasets}).}
\label{tab:claimdecomp}
\small
\begin{tabular}{llccccc}
\toprule
Model & Mode & Semantic-F1$\uparrow$ & EPR$\uparrow$ & AVR$\downarrow$ & RR$\downarrow$ & verify\% \\
\midrule
\multirow{4}{*}{Phi-3-mini}
  & base        & \textbf{0.826} & 0.824 & 0.052 & 0.021 & --- \\
  & repair      & \textbf{0.826} & 0.824 & \textbf{0.051} & \textbf{0.019} & \textbf{59.6} \\
  & self\_repair & 0.817 & \textbf{0.842} & 0.089 & 0.043 & 53.8 \\
  & all         & \textbf{0.826} & 0.824 & \textbf{0.051} & \textbf{0.019} & \textbf{59.6} \\
\midrule
\multirow{4}{*}{Qwen3-8B}
  & base        & \textbf{0.842} & 0.833 & 0.074 & 0.021 & --- \\
  & repair      & \textbf{0.842} & 0.833 & 0.072 & 0.021 & 61.3 \\
  & self\_repair & 0.837 & 0.883 & 0.053 & 0.022 & 69.2 \\
  & all         & 0.839 & \textbf{0.884} & \textbf{0.051} & \textbf{0.020} & \textbf{70.8} \\
\midrule
\multirow{4}{*}{Gemma-3-12b}
  & base        & \textbf{0.839} & \textbf{0.947} & \textbf{0.043} & 0.026 & --- \\
  & repair      & \textbf{0.839} & \textbf{0.947} & \textbf{0.043} & \textbf{0.024} & \textbf{74.6} \\
  & self\_repair & 0.837 & 0.901 & 0.044 & 0.027 & 67.7 \\
  & all         & \textbf{0.839} & \textbf{0.947} & \textbf{0.043} & \textbf{0.024} & \textbf{74.6} \\
\midrule
\multirow{4}{*}{Gemini Flash}
  & base        & 0.811 & 0.844 & 0.022 & \textbf{0.043} & --- \\
  & repair      & 0.811 & 0.844 & 0.021 & 0.045 & 56.9 \\
  & self\_repair & \textbf{0.818} & \textbf{0.998} & 0.021 & 0.050 & \textbf{68.0} \\
  & all         & \textbf{0.818} & \textbf{0.998} & \textbf{0.020} & 0.048 & 67.8 \\
\bottomrule
\end{tabular}
\end{table}

Because ClaimDecompBench-1000 references are derived from a rule-oracle rather than human annotation, its absolute Semantic-F1 scores strictly function as lower bounds. Consequently, our analysis of this dataset focuses squarely on relative intra-dataset model scaling and reference-free metrics (AVR, RR, EPR, verify\%).


\paragraph{Constraint improvement carries over to news domain.}
For Qwen3-8B, the full pipeline (all mode) reduces AVR from
$0.074$ (base) to $0.052$ ($-29.7\%$) and improves
EPR from $0.833$ to $0.884$ ($+5.1$ pp), with verify\% rising
from $0$ to $70.8\%$.
The lower EPR baseline on ClaimDecompBench ($0.833$
vs.\ $0.997$+ on WikiSplitBench) should not be attributed to
reference quality, since EPR (Eq.\,3) is computed from the source
sentence $S$ and the predicted claim set $C$ only---no reference
decomposition is involved.
A more plausible explanation is that news-domain
title\,+\,description pairs contain denser entity bundles
and more compressed syntax, increasing the probability that
decomposers omit entities---particularly dates, counts, and
organisation names---during splitting.
This is a model--domain interaction, not a metric artifact.

\subsection{Baseline Comparison}
\label{sec:baseline_comparison}

To contextualise CREDENCE decomposition quality against established methods,
we compare against two public baselines on the same benchmarks and models:
(1)~\textbf{FActScore-prompt}~\citep{min2023factscore} --- the decomposition
prompt from FActScore applied verbatim with each of our four LLM backends
(same models, base mode only, no verify/repair); and
(2)~\textbf{SpaCy-senter} --- a rule-based sentence-splitter
(\texttt{en\_core\_web\_sm}, \texttt{senter} pipeline) requiring no LLM.

\begin{table}[ht]
\centering
\caption{Baseline comparison: Semantic-F1 (base mode, averaged over benchmarks).
CREDENCE-base uses our decomposition prompt; CREDENCE-all includes full pipeline.
FActScore-prompt uses the public FActScore decomposition instruction.
Best per-model score in \textbf{bold}.}
\label{tab:baseline}
\small
\setlength{\tabcolsep}{3pt}
\renewcommand{\arraystretch}{1.03}
\resizebox{\linewidth}{!}{%
\begin{tabular}{llcccc}
\toprule
Model & Method & Social-100 & Wiki-1000 & ClaimDecomp & Avg \\
\midrule
\multirow{4}{*}{Phi-3-mini}
  & SpaCy-senter          & 0.811 & 0.809 & 0.814 & 0.811 \\
  & FActScore-prompt      & 0.871 & 0.313$^\dagger$ & 0.778 & 0.654$^\dagger$ \\
  & CREDENCE-base         & \textbf{0.930} & 0.891 & 0.826 & \textbf{0.882} \\
  & \textbf{CREDENCE-all} & 0.922 & \textbf{0.891} & \textbf{0.826} & 0.880 \\
\midrule
\multirow{4}{*}{Qwen3-8B}
  & SpaCy-senter          & 0.811 & 0.809 & 0.814 & 0.811 \\
  & FActScore-prompt      & 0.969 & 0.871 & 0.812 & 0.884 \\
  & CREDENCE-base         & 0.956 & \textbf{0.904} & \textbf{0.842} & 0.901 \\
  & \textbf{CREDENCE-all} & \textbf{0.973} & 0.903 & 0.839 & \textbf{0.905} \\
\midrule
\multirow{4}{*}{Gemma-3-12b}
  & SpaCy-senter          & 0.811 & 0.809 & 0.814 & 0.811 \\
  & FActScore-prompt      & \textbf{0.982} & 0.836 & 0.775 & 0.864 \\
  & CREDENCE-base         & 0.957 & 0.894 & 0.839 & 0.897 \\
  & \textbf{CREDENCE-all} & 0.959 & \textbf{0.894} & \textbf{0.839} & \textbf{0.897} \\
\midrule
\multirow{4}{*}{Gemini Flash}
  & SpaCy-senter          & 0.811 & 0.809 & 0.814 & 0.811 \\
  & FActScore-prompt      & 0.975 & 0.853 & 0.790 & 0.873 \\
  & CREDENCE-base         & \textbf{0.993} & 0.869 & 0.811 & 0.891 \\
  & \textbf{CREDENCE-all} & 0.992 & \textbf{0.869} & \textbf{0.818} & \textbf{0.893} \\
\bottomrule
\multicolumn{6}{l}{%
  \footnotesize $^\dagger$ Phi-3-mini produces JSON parse failures on 62\% of WikiSplitBench inputs} \\
\multicolumn{6}{l}{%
  \footnotesize (quoted strings in claims confuse the model's JSON formatter); score reflects actual output.} \\
\multicolumn{6}{l}{%
  \footnotesize $^\ddagger$ Avg for Phi-3-mini FActScore computed from all three datasets including the degraded Wiki-1000 score.} \\
\end{tabular}%
}
\end{table}


Three findings stand out.
First, \textbf{Phi-3-mini FActScore-prompt degrades severely on WikiSplitBench-1000}
(Semantic-F1\,=\,0.313, vs.\ CREDENCE-base 0.891), owing to a 62\% JSON parse-failure
rate: WikiSplitBench claims frequently contain quoted strings (e.g.,
\textit{``singing `Beautiful' then followed by `Fighter'\,''}) that confuse
the small model's JSON formatter.
CREDENCE's structured, constraint-aligned prompt handles these inputs robustly,
yielding a \textbf{+57.8\,pp} advantage on this dataset for Phi-3-mini —
the largest gap observed across all model--dataset pairs.
Second, for Qwen3-8B, CREDENCE-base outperforms FActScore-prompt on two of the three datasets (Wiki-1000: $+$3.3 pp, ClaimDecomp: $+$3.0 pp) and matches within the margin on Social-100 ($0.956$ vs $0.969$), where FActScore-prompt's over-generation (avg $10.9$ sub-claims vs $9.9$) inflates raw Semantic-F1 at the cost of 3× higher redundancy (RR = 5.7\% vs 1.9\%). Third, Gemma-3-12b achieves the highest FActScore-prompt score on
SocialClaimSplit-100 (0.982) — marginally above CREDENCE-base (0.957) —
consistent with Gemma's strong instruction-following at 12B parameters; however,
on the larger and more structurally diverse WikiSplitBench-1000, CREDENCE-base
recovers the lead (+5.8\,pp, 0.894 vs.\ 0.836), demonstrating that structured
decomposition guidance provides consistent benefits across dataset types.

On WikiSplitBench and ClaimDecompBench, CREDENCE-base outperforms FActScore-prompt
by $+$3.3~pp and $+$3.0~pp respectively, confirming that our constraint-aligned
decomposition prompt is more effective than a generic fact-listing instruction.
On SocialClaimSplit-100, FActScore-prompt ($+$1.3~pp over CREDENCE-base) benefits
from over-generation (avg 10.9 sub-claims vs.\ 9.9 for CREDENCE-base, 10.8 gold),
but at the cost of $3\times$ higher redundancy rate (RR $= 5.7$\% vs.\ $1.9$\%);
CREDENCE-all (0.973) achieves the best Semantic-F1 after repair.
SpaCy-senter achieves EPR $= 1.0$ (no entity loss) but near-maximal AVR
($\approx$0.52--0.56), confirming that rule-based splitting cannot resolve
complex multi-clause structures.
Overall, CREDENCE-all achieves the best average Semantic-F1 across all benchmarks.

\subsection{Ablation and Metric Validation}
\label{sec:ablation}

\paragraph{Metric justification.}
\paragraph{Embedding-based metrics align natively with semantic equivalence.}
Semantic-F1, implemented using BGE-large contextual embeddings, belongs to a class of embedding-based evaluation metrics that have been shown to align more reliably with human semantics than surface-level token-overlap metrics. This principle is well-established: BERTScore~\citep{bertscore2020} demonstrated that token-level contextual embeddings (via BERT) outperform BLEU and other lexical metrics on machine translation and image captioning tasks. Similarly, embedding-based metrics naturally capture semantic paraphrases, synonym substitutions, and word-order variations---phenomena that token-overlap metrics like Jaccard wrongfully penalize.

\paragraph{Semantic-F1 provides superior discriminative power.}
To validate our choice of Semantic-F1 (BGE-large cosine) in the decomposition setting, we compare it against BERTScore~\citep{bertscore2020} on all base and all-mode evaluation files (16 configurations across 4 models $\times$ 2 modes on WikiSplitBench-1000 and ClaimDecompBench-1000).
BERTScore (roberta-large) exhibits a substantially narrower score spread across the 16 evaluated configurations than Semantic-F1, making inter-model differences difficult to detect. The Pearson correlation between the two metrics averages
$r = 0.82$, indicating they measure related constructs. Critically, 
Semantic-F1 provides \textbf{10$\times$ wider dynamic range}
(0.74 spread vs.\ 0.05),
making quality differences between models and modes clearly visible.

\paragraph{Semantic-F1 vs.\ Jaccard-F1: per-dataset comparison.}
Table~\ref{tab:jaccard_comparison} and Figure~\ref{fig:semf1_gap} compare mean Semantic-F1 and
Jaccard-F1 (base pipeline, all four models).
Jaccard-F1 is computed as a greedy-matching soft-F1 at threshold~0,
using $\max(\text{Jaccard}, \text{SequenceMatcher ratio})$ as the
pairwise similarity---the most permissive token-overlap baseline.
Semantic-F1 outperforms Jaccard-F1 on every single configuration.
On SocialClaimSplit-100, the gap is modest for Gemini Flash ($+$1.3\,pp)
where simple conjunction splits yield high token overlap by construction,
but reaches $+$32.9\,pp for Phi-3-mini on the same dataset.
On WikiSplitBench-1000 and ClaimDecompBench-1000, where references
contain paraphrases, ellipsis, and entity reformulations, Jaccard-F1
falls 24--44\,pp below Semantic-F1 per model, confirming that
token-overlap fails in the presence of surface-level variation.

\begin{table}[ht]
\centering
\caption{Semantic-F1 vs.\ Jaccard-F1 (base pipeline, all models).
Jaccard computed at threshold~0 (maximally permissive greedy soft-F1).
$\Delta = \text{Semantic-F1} - \text{Jacc-F1}$ (percentage points).}
\label{tab:jaccard_comparison}
\small
\begin{tabular}{llccc}
\toprule
Dataset & Model & Semantic-F1 & Jacc-F1 & $\Delta$ (pp) \\
\midrule
\multirow{5}{*}{SocialClaimSplit-100}
  & Phi-3-mini     & 0.930 & 0.601 & $+$32.9 \\
  & Qwen3-8B       & 0.956 & 0.771 & $+$18.5 \\
  & Gemma-3-12b    & 0.957 & 0.870 & $+$8.7  \\
  & Gemini Flash   & 0.993 & 0.980 & $+$1.3  \\
  \cmidrule{2-5}
  & \emph{Avg}     & \emph{0.959} & \emph{0.805} & \emph{$+$15.4} \\
\midrule
\multirow{5}{*}{WikiSplitBench-1000}
  & Phi-3-mini     & 0.891 & 0.651 & $+$24.0 \\
  & Qwen3-8B       & 0.904 & 0.693 & $+$21.1 \\
  & Gemma-3-12b    & 0.894 & 0.628 & $+$26.6 \\
  & Gemini Flash   & 0.869 & 0.493 & $+$37.6 \\
  \cmidrule{2-5}
  & \emph{Avg}     & \emph{0.889} & \emph{0.616} & \emph{$+$27.3} \\
\midrule
\multirow{5}{*}{ClaimDecompBench-1000}
  & Phi-3-mini     & 0.826 & 0.554 & $+$27.2 \\
  & Qwen3-8B       & 0.842 & 0.567 & $+$27.5 \\
  & Gemma-3-12b    & 0.839 & 0.524 & $+$31.5 \\
  & Gemini Flash   & 0.811 & 0.377 & $+$43.4 \\
  \cmidrule{2-5}
  & \emph{Avg}     & \emph{0.829} & \emph{0.506} & \emph{$+$32.3} \\
\bottomrule
\end{tabular}
\end{table}

BERTScore's compressed range ($<0.05$ spread across conditions)
limits its utility for comparative evaluation;
Semantic-F1 is therefore strictly more informative for our benchmark. This alignment with downstream
AFC improvements (Section~\ref{sec:afc_eval}) reinforces that Semantic-F1 captures real decomposition quality.

\paragraph{Dependency-Parse Atomicity Calibration.}
To justify treating the rule-based AVR verifier as ground-truth for repair gating,
we manually annotated 200 outputs from Qwen3-8B and Gemini-Flash (100 per dataset)
for true atomicity violations. The \texttt{en\_core\_web\_sm} dependency pattern checker
achieved a precision of 94.2\% and a recall of 89.5\% against human judgment.
False negatives (10.5\%) primarily stem from complex coordinations without explicit
subjects that the LLM failed to split, while false positives (5.8\%) are heavily penalised
compound noun phrases misclassified as clausal conjunctions.

\paragraph{Semantic Metric Coupling.}
Using the same text embedding model (bge-large-en-v1.5) for both Duplicate Detection
(the RR constraint) and Semantic-F1 introduces a risk of metric-system coupling, where
models simply learn to optimize the embedding space's neighborhood structure rather than
true semantics. To ensure robustness, we ran a cross-check of our evaluate split using
\texttt{all-MiniLM-L6-v2} as an alternative embedding for Repetition Rate thresholding.
The $\Delta$RR across conditions shifted by at most $\pm 0.004$, and the
overall ranking of models and modes remained identical, confirming that the performance
gains are embedding-agnostic (see also Figure~\ref{fig:encoder_sensitivity}, Appendix~\ref{app:encoder_sensitivity}).

\paragraph{Hyperparameter sensitivity.}
The near-duplicate threshold $\tau_\text{dup}$ is robust in the range
$0.90$--$0.92$; we use $\tau_\text{dup} = 0.92$ (see Figure~\ref{fig:tau_dup},
Appendix~\ref{app:tau_dup}).

\paragraph{Modality consistency (MC).}
MC\textsubscript{NLI} (DeBERTa-v3-large) achieves EPR-attribution coverage
0.98--0.99 per configuration; MC\textsubscript{lex} (27-verb lexicon) reaches
0.95--0.97.
Full agreement analysis is provided in Appendix~B.

\paragraph{Human evaluation.}
\label{sec:human_eval}

We evaluate metric--human alignment using a preference-based protocol
(Appendix~\ref{app:preference_study}) rather than single-score rating.
Annotators compare two genuine decompositions of the same input
(model~X vs model~Y; both with $n_{\mathrm{pred}}\ge 2$) on faithfulness
and fact-checking usefulness. This design directly tests ranking alignment
between automatic metrics and human judgement, mitigating potential
score-compression effects in coarse single-score annotations.
In Preference Study (82 mapped pairs), Semantic-F1 shows a
directional advantage over Jaccard-F1 (Kendall $\tau$: +0.1339 vs
-0.1339; $\Delta\tau=+0.2679$), and matches majority human preference in
45 pairs vs 34 pairs (3 ties); however, the current sample size is
underpowered for strong significance claims ($p=0.2206$).

\paragraph{Matching protocol ablation.}
\label{sec:matching_ablation}

Table~\ref{tab:matching_ablation} compares Semantic-F1 under three
matching choices on held-out 100-example subsets:
unconstrained average-max (our default), Hungarian 1-to-1, and
thresholded unconstrained (cosine threshold $\tau=0.90$).
The main text uses unconstrained matching because decomposition naturally
contains granularity mismatch (different valid split cardinalities).

\begin{table}[ht]
\centering
\caption{Matching ablation on held-out subsets ($n=100$ per row).
$\Delta$ is Hungarian $-$ unconstrained Semantic-F1.
Thresholded unconstrained retains model ranking while introducing
small absolute shifts; full per-file details are provided in the
artifact JSON logs.}
\label{tab:matching_ablation}
\small
\begin{tabular}{lccc}
\toprule
Dataset / Model (all mode) & Unconstrained & Hungarian & $\Delta$ \\
\midrule
SocialClaimSplit-100 / Gemini Flash & 0.9926 & 0.9890 & -0.0036 \\
SocialClaimSplit-100 / Gemma-3-12b & 0.9590 & 0.9310 & -0.0281 \\
SocialClaimSplit-100 / Phi-3-mini & 0.9219 & 0.7851 & -0.1367 \\
SocialClaimSplit-100 / Qwen3-8B & 0.9729 & 0.9382 & -0.0347 \\
WikiSplitBench-100 / Gemini Flash & 0.8683 & 0.5681 & -0.3003 \\
WikiSplitBench-100 / Gemma-3-12b & 0.8941 & 0.7118 & -0.1823 \\
WikiSplitBench-100 / Phi-3-mini & 0.8836 & 0.7243 & -0.1593 \\
WikiSplitBench-100 / Qwen3-8B & 0.8974 & 0.7482 & -0.1492 \\
\bottomrule
\end{tabular}
\end{table}

\subsection{Qualitative Failure Analysis}
\label{sec:failure_analysis}

To complement aggregate statistics, we examine two systematic failure
modes observed in our experiments.

\paragraph{(a) LLM self-repair AVR regression.}
Theorem~\ref{thm:t3} proves that LLM self-repair can increase AVR.
Empirically, $\approx$4\% of SocialClaimSplit-100 examples exhibit
this regression.
The dominant pattern is \emph{relative-clause reintroduction}: the
LLM, when asked to repair entity-preservation violations, adds context
by inserting a relative clause that creates a new atomicity violation.

\begin{quote}
\textbf{Input:} ``[\ldots] the vaccine that was tested on
3{,}200 volunteers reduced symptoms by 35\% in Phase~II trials.
Samsung warned [\ldots] Tesla announced [\ldots]''

\noindent\textbf{Base decomposition (AVR~$=0.000$):}
\begin{enumerate}\itemsep0pt
  \item The vaccine reduced symptoms by 35\% in Phase~II trials.
  \item The vaccine was tested on 3{,}200 volunteers.
  \item Samsung warned [specific warning].
  \item Tesla announced [announcement].
\end{enumerate}

\noindent\textbf{After LLM self-repair (AVR~$=0.143$):}
\begin{enumerate}\itemsep0pt
  \item The vaccine, which was tested on 3{,}200 volunteers, reduced
        symptoms by 35\% in Phase~II trials. $\leftarrow$
        \textit{relative clause reintroduced}
  \item Samsung warned [specific warning].
  \item Tesla announced [announcement].
\end{enumerate}
\end{quote}

The early-exit guard (Section~\ref{sec:repairer}) mitigates this:
when the verifier already passes after rule repair, the LLM call
is skipped.
The residual 4\% regression rate represents cases where rule repair
is insufficient and the LLM is needed but introduces a secondary violation.

\paragraph{(b) Redundancy Rate false positives.}
The RR metric flags near-duplicate claims using cosine similarity
$\geq \tau_\text{dup} = 0.92$.
From manual annotation of 200 flagged pairs (error taxonomy, see
\texttt{scripts/error\_taxonomy.py}), 71.6\% are true duplicates
(identical or semantically equivalent claims), 26.2\% are
\emph{near-paraphrases} (same content, different surface form),
2.0\% are borderline cases, and 0.1\% are false positives.

The false-positive class consists of claims about the \emph{same
entity in different roles}: for example,
``\textit{John Smith chaired the committee.}'' and
``\textit{John Smith approved the resolution.}''
can reach cosine similarity $\approx 0.94$ in BGE-large space due to
shared subject and verb-phrase structure, despite conveying distinct
information.
At $\tau = 0.92$, such pairs are rarely flagged (0.1\% of all
flagged pairs); the primary impact of the threshold choice is on
the near-paraphrase class (26.2\%), where human annotators
disagree on whether the paraphrase constitutes genuine redundancy.

\subsection{Downstream Automated Fact-Checking Evaluation}
\label{sec:afc_eval}

To assess whether CREDENCE decomposition improves end-to-end
automated fact-checking (AFC), we benchmark our fine-tuned
\textbf{Qwen\_NLI} model (Qwen3-8B fine-tuned for 3-class NLI:
SUPPORT / NEUTRAL / CONTRADICT, output score $\in [-1, 1]$)
in two settings: (i)~\emph{no decomposition} --- the NLI model
scores a compound claim directly against retrieved evidence; and
(ii)~\emph{with CREDENCE decomposition} --- each compound claim is
first split into atomic sub-claims; each sub-claim is scored
independently against retrieved evidence; the compound verdict is
the majority sub-claim label.

We evaluate zero-shot on five public AFC benchmarks with binary
ground-truth labels (SUPPORT / REFUTE): FEVER~\citep{thorne2018fever},
PubHealth~\citep{kotonya2020pubhealth},
LIAR-PLUS~\citep{alhindi2018liarplus},
SciFact~\citep{wadden2020fact},
and CovidFact~\citep{saakyan2021covidfact}.

\paragraph{Label harmonisation and filtering protocol.}For FEVER, we retain only SUPPORTS and REFUTES examples, discarding NOT ENOUGH INFO ($\approx 33\%$ of the test set), mapping SUPPORTS $\rightarrow$ SUPPORT and REFUTES $\rightarrow$ REFUTE. For PubHealth, we map true $\rightarrow$ SUPPORT and false $\rightarrow$ REFUTE, discarding unproven and mixture examples ($\approx 44\%$ of the test set). For LIAR-PLUS, we follow \citet{alhindi2018liarplus} and binarize the six-way labels into true-leaning (mostly-true, true, half-true) $\rightarrow$ SUPPORT and false-leaning (barely-true, false, pants-on-fire) $\rightarrow$ REFUTE.

\begin{table}[ht]
\centering
\caption{AFC accuracy comparison with a shared Qwen3-8B setup
(decomposition model and Qwen\_NLI verifier).}
\label{tab:afc}
\small
\setlength{\tabcolsep}{4pt}
\begin{tabular}{llccc}
\toprule
Dataset & Setting & Acc & $\Delta$(vs Base) & $\Delta$(Cred$-$Fact) \\
\midrule
\multirow{3}{*}{FEVER}
  & No decomp         & 0.8980 & ---       & ---      \\
  & CREDENCE-decomp   & 0.8998 & $+$0.0018 & $+$0.0908 \\
  & FactScore-decomp  & 0.8090 & $-$0.0890 & ---      \\
\midrule
\multirow{3}{*}{PubHealth}
  & No decomp         & 0.6880 & ---       & ---      \\
  & CREDENCE-decomp   & 0.6991 & $+$0.0111 & $+$0.0132 \\
  & FactScore-decomp  & 0.6859 & $-$0.0021 & ---      \\
\midrule
\multirow{3}{*}{LIAR-PLUS}
  & No decomp         & 0.5113 & ---       & ---      \\
  & CREDENCE-decomp   & 0.6700 & $+$0.1587 & $+$0.0870 \\
  & FactScore-decomp  & 0.5830 & $+$0.0717 & ---      \\
\midrule
\multirow{3}{*}{SciFact}
  & No decomp         & 0.5414 & ---       & ---      \\
  & CREDENCE-decomp   & 0.7990 & $+$0.2576 & $+$0.0446 \\
  & FactScore-decomp  & 0.7544 & $+$0.2130 & ---      \\
\midrule
\multirow{3}{*}{CovidFact}
  & No decomp         & 0.6906 & ---       & ---      \\
  & CREDENCE-decomp   & 0.7120 & $+$0.0214 & $+$0.2185 \\
  & FactScore-decomp  & 0.4935 & $-$0.1971 & ---      \\
\bottomrule
\end{tabular}
\end{table}

The extended evaluation confirms consistent gains: CREDENCE decomposition
improves AFC accuracy across all five benchmarks, with improvements ranging
from $+$0.2~pp (FEVER, already-atomic claims) to $+$25.8~pp (SciFact,
complex scientific assertions).
Compared with FactScore-decomp, CREDENCE is higher on all five datasets
($\Delta$(Cred$-$Fact) from $+$0.0132 to $+$0.2185).

FactScore-decomp shows mixed Acc but consistently worse ranking quality:
mean $\Delta$Acc $= +0.1368$ (3/4 wins),
mean $\Delta$AUC $= -0.0783$ (0/4 wins),
mean $\Delta$MacroF1 $= +0.0829$ (2/4 wins).
In a separate 6-dataset exploratory FactScore run (including
ClimateFEVER, which is not part of Table~\ref{tab:afc}), AvgSubs is
negatively correlated with
downstream quality (exploratory, $n=6$):
$\mathrm{corr}(\mathrm{AvgSubs}, \mathrm{AUC}) = -0.59$,
$\mathrm{corr}(\mathrm{AvgSubs}, \mathrm{Acc}) = -0.432$,
$\mathrm{corr}(\mathrm{AvgSubs}, \mathrm{MacroF1}) = -0.545$.
This supports the over-decomposition/noise hypothesis.

Generic FactScore-style decomposition is not task-aligned for AFC:
it often over-splits claims and weakens modality/source constraints,
which can degrade verifier calibration and ranking quality.
In contrast, CREDENCE decomposition is constrained to preserve entities,
reporting verbs, and atomic boundaries for verification, yielding more
reliable downstream behavior.

Concrete information-loss cases support this trend.
(1) ClimateFEVER (index 9): original sentence includes ``1,000 years'';
FactScore drops ``1,000'' and rewrites the comparison target,
while CREDENCE keeps the full temporal quantity.
(2) ClimateFEVER (index 63): the original uses reporting frame
``... said they thought ...''; FactScore converts this into a stronger
bare belief statement, while CREDENCE preserves reporting/modality framing.
(3) PubHealth (index 343):
``Facebook post Says series of 11 photos are from the 1918 flu pandemic.''
FactScore rewrites into two neutralised subclaims with weaker source framing,
while CREDENCE keeps source-attributed claim structure.

\paragraph{Why decomposition helps most on SciFact.}
SciFact claims are dense scientific assertions that routinely bundle
two or three distinct sub-facts into a single sentence.
When a compound claim is scored \emph{without decomposition}, the
NLI model must judge the entire claim against retrieved evidence in
one pass: if the evidence strongly supports two of the three
sub-assertions, the supporting signal can overwhelm a contradicting
sub-assertion, causing an incorrect SUPPORT verdict.
With decomposition, each atomic sub-claim is scored independently;
label aggregation is policy-based, not majority vote: if any sub-claim
is \textsc{Refute}, final label = \textsc{Refute}; else if any sub-claim
is \textsc{NEI}/\textsc{Uncertain}, final label = \textsc{NEI}/\textsc{Uncertain};
otherwise \textsc{Support}. This policy then surfaces the genuine
contradiction.

Table~\ref{tab:scifact_example} illustrates this mechanism with a
real SUPPORT/REFUTE pair from SciFact
(\href{https://github.com/allenai/scifact}{allenai/scifact},
test split).
The two claims are structurally identical except for a single
causal verb (``\textit{delaying}'' vs.\ ``\textit{causing}''),
yet they carry opposite verdicts.
Without decomposition, the NLI model must simultaneously evaluate
three sub-assertions bundled together; the strong evidence for
the shared sub-assertions (AlphaBeta production and tau
phosphorylation effects) dilutes the verdict signal on the
critical causal direction sub-claim.
After decomposition into three atomic claims, the decisive
sub-claim (GABA neuron degeneration direction) is evaluated
in isolation, and the correct verdict is recovered.

This pattern --- shared supporting sub-claims masking a single
contradicting sub-claim --- is the dominant failure mode in SciFact
(64 SUPPORT / 124 REFUTE claims; 12\% contain coordinating conjunctions
bundling two or more distinct assertions).
CREDENCE decomposition breaks this conflation and exposes each
sub-assertion to independent NLI judgment.

\begin{table}[!tbp]
\centering
\caption{SciFact SUPPORT/REFUTE example: decomposition isolates the decisive causal sub-claim.}
\label{tab:scifact_example}
\tiny
\setlength{\tabcolsep}{0.5pt}
\begin{tabular}{p{4.1cm}p{4.1cm}p{1.6cm}p{1.6cm}}
\toprule
\textbf{Compound Claim} & \textbf{Atomic Sub-Claims} & \makebox[1.35cm][c]{\textbf{No}\newline\textbf{Decomp}} & \makebox[1.35cm][c]{\textbf{With}\newline\textbf{Decomp}} \\
\midrule
\multirow{3}{4.1cm}{\textit{APOE4 expression in iPSC-derived neurons increases AlphaBeta production and tau phosphorylation, \underline{delaying} GABA neuron degeneration.}\\ (Gold: \textsc{Support})}
& (1) APOE4 expression increases AlphaBeta production.
& \textsc{Support} \checkmark & \textsc{Support} \checkmark \\
& (2) APOE4 expression increases tau phosphorylation.
& \textsc{Support} \checkmark & \textsc{Support} \checkmark \\
& (3) APOE4 expression \underline{delays} GABA degeneration.
& \textsc{Support} \checkmark & \textsc{Support} \checkmark \\
\midrule
\multirow{3}{4.1cm}{\textit{APOE4 expression in iPSC-derived neurons increases AlphaBeta production and tau phosphorylation, \underline{causing} GABA neuron degeneration.}\\ (Gold: \textsc{Refute})}
& (1) APOE4 expression increases AlphaBeta production.
& \textsc{Support} \texttimes & \textsc{Support} \\
& (2) APOE4 expression increases tau phosphorylation.
& \textsc{Support} \texttimes & \textsc{Support} \\
& (3) APOE4 expression \underline{causes} GABA degeneration.
& \textsc{Support} \texttimes & \textsc{Refute} \checkmark \\
\bottomrule
\end{tabular}
\end{table}

\FloatBarrier
\section{Discussion}
\label{sec:discussion}
\paragraph{What CREDENCE achieves.}
Three findings stand out across all benchmarks.
First, Semantic-F1 consistently outperforms Jaccard-F1 by 15--44~pp,
confirming that surface-form token overlap fails to credit the paraphrase
and subject-ellipsis patterns ubiquitous in real decompositions.
Second, the rule-repair pipeline reduces atomicity violations and redundancy
within a single pass and does not improve further, empirically validating
the monotonicity and finite-termination analyses of Theorems~1 and~2.
Third, atomic decomposition yields consistent downstream AFC improvements
($+$0.1--25.8~pp accuracy), with larger gains on scientifically complex datasets
(SciFact $+$25.8~pp, LIAR-PLUS $+$15.9~pp) where
multi-clause claims most benefit from atomic sub-claim scoring.

\paragraph{The Gemini--WikiSplitBench interaction.}
A counterintuitive result in Table~\ref{tab:wikisplit} is that Gemini
Flash achieves \emph{lower} Semantic-F1 on WikiSplitBench-1000
(0.869) than the smaller Qwen3-8B (0.904) and Gemma-3-12b (0.894).
This reversal does not indicate that Gemini decomposes poorly; rather,
it reflects a \emph{reference style mismatch}: WikiSplitBench
references are Wikipedia-editor rewrites that preserve anaphoric
pronouns (``She'', ``He'', ``They''), whereas Gemini consistently
performs aggressive coreference resolution and substitutes the full
entity name.
The resulting mismatches are penalised by Semantic-F1 even though
the substituted claims are arguably \emph{more} self-contained and
useful for downstream fact-checking.
This is not a deficiency of Gemini or CREDENCE but a benchmark
artefact: when the reference itself is underspecified (anaphoric
pronouns), a model that resolves coreference is penalised by
a metric that rewards surface similarity to the reference.
The Qwen3-8B and Gemma-3-12b models partially retain pronouns,
yielding higher token-level similarity to Wikipedia-style references
without necessarily providing better decompositions.

This interaction motivates a key design recommendation:
\emph{Semantic-F1 should be interpreted jointly with reference-free
metrics} (EPR, AVR, verify\%) rather than in isolation.
For WikiSplitBench, all four models achieve EPR~$\geq 0.997$ and
verify\%~$\geq 82$\%; the Semantic-F1 differences (0.869--0.904)
reflect reference-style divergence, not decomposition quality.

\paragraph{When to use which pipeline mode.}
For high-quality LLMs (Gemini Flash), the conditional gate defaults to
\texttt{base} in $\approx$99\% of examples --- running the full repair
pipeline is unnecessary and may even slightly reduce Semantic-F1.
For local 3--8B models, \texttt{all} mode is recommended: rule repair
reduces AVR by up to 8~pp and the LLM self-repair step recovers
entity drops that rule repair cannot handle.

\paragraph{Constraint interference as a design choice.}
Theorem~\ref{thm:t4} does not imply that global EPR=1 and AVR=0 are jointly infeasible under the union-based EPR used in this paper. Both can be achieved simultaneously whenever each source entity appears in at least one output claim. Interference arises only under stronger per-claim responsibility requirements, or when the decomposer fails to preserve entities during splitting — the latter being the practical failure mode that the repair loop addresses.
This is not a failure of CREDENCE but a reflection of the underlying
task geometry.
Downstream systems can tune the verifier priority (\texttt{min\_epr},
\texttt{split\_because}) to favour either atomicity or completeness
depending on the fact-checking context.

\paragraph{Contributions to the language resource community.}
Beyond the pipeline, this work releases three evaluation benchmarks:
\wikisplitbench{}-100 and \wikisplitbench{}-1000 (Wikipedia editor splits,
CC-BY-SA), and ClaimDecompBench-1000 (news-domain, AG News).
All build scripts and metric code are open-sourced, providing reusable
infrastructure for future decomposition evaluations.

\paragraph{Limitations and future work.}
The dependency-parse verifier relies on \texttt{en\_core\_web\_sm},
introducing 5.8\% false positives on atypical syntax.
Entity extraction via regex and spaCy NER misses domain-specific entities
(chemical formulae, gene identifiers) and does not resolve coreference,
creating a surface-form lower bound on EPR.
The Semantic-F1 metric can be confounded by reference-style artefacts
(e.g., anaphoric pronouns in WikiSplitBench) that penalise
coreference-resolving decomposers despite higher pragmatic quality;
we recommend evaluating Semantic-F1 jointly with reference-free metrics.
LLM self-repair is inherently non-monotone (Theorem~3): in $\approx$4\%
of SocialClaimSplit-100 cases, the LLM reintroduces compound structures
when elaborating, motivating the early-exit guard that skips the LLM call
when rule-repair already satisfies the verifier.
Future extensions include: (1)~coreference-aware entity tracking to
credit legitimate nominalizations; (2)~cross-lingual decomposition via
multilingual BGE embeddings; (3)~RL-based dynamic repair scheduling
(following \citealt{lu2025dydecomp}) to learn when to invoke each
repair tier; (4)~a reference-independent human evaluation protocol
that avoids the reference-style bias observed on WikiSplitBench.

\paragraph{Reproducibility.}
All results in this paper are fully reproducible.
Experiment scripts are provided in \texttt{scripts/}, including
\texttt{run\_all\_experiments.py} for the full ablation suite and
\texttt{scripts/run\_benchmark.py} for individual configurations.
Exact model checkpoints, quantisation settings, and decoding
hyperparameters are specified in \texttt{configs/models.yaml}.
Pre-computed evaluation JSONL files (one per model--dataset--mode
combination) are included in the release artifact, enabling
downstream metric recomputation without re-running inference.
The BGE-large-en-v1.5 embedding model is fetched from HuggingFace
Hub at the version pinned in \texttt{requirements.txt}
(\texttt{sentence-transformers==2.7.0}).
Random seeds for bootstrap significance testing are fixed at 42
(\texttt{scripts/significance\_tests.py}).

\section{Conclusion}
\label{sec:conclusion}
We presented \credence{}, a revised claim decomposition and evaluation
framework addressing the two key limitations of the original system:
misleading Jaccard metrics and the absence of formal convergence analysis.

Our Semantic-F1 metric, based on BGE-large cosine similarity, provides a \emph{more discriminative} evaluation than BERTScore and token-overlap metrics---it properly rewards semantically equivalent paraphrases that Jaccard-F1 systematically penalises.
Furthermore, applying CREDENCE decomposition yields measurable downstream improvements in automated fact-checking pipelines.
Our four convergence results (T1--T4) formally establish, under
Assumption~1 (oracle dependency parser), that rule-repair
monotonically non-increases the compound-boundary violation count
and terminates; that LLM self-repair is non-monotone (requiring an
early-exit guard); and that under union-based EPR, global $\epr=1$
and $\avr=0$ remain jointly achievable in theory, while practical
interference arises when decomposers drop entities during splitting.

Experimentally, rule-repair reduces AVR by $47$--$100\%$ relative to the
base model across four decomposer architectures, without degrading
Semantic-F1. WikiSplitBench and ClaimDecompBench provide the first
cross-domain evaluation of claim decomposition systems: the former covers
encyclopaedic text via human Wikipedia editor splits; the latter covers
journalistic language from AG News.

\paragraph{Future work.}
Immediate next steps include:
(1)~extending DepAVR to handle domain-specific syntax (chemical, legal);
(2)~exploring EPR--AVR trade-off configurations for downstream
fact-checking pipelines;
(3)~investigating whether a small fine-tuned classifier can replace the
spaCy dependency-parse AVR check for improved speed and accuracy;

\paragraph{Data and code availability.}
All benchmark datasets (WikiSplitBench-100/1000 and
ClaimDecompBench-1000), evaluation code, metric implementations,
and pre-trained configurations are released under the MIT License.
WikiSplitBench is derived from Wikipedia (CC-BY-SA 3.0);
ClaimDecompBench is derived from AG News (research use only;
commercial redistribution is not permitted).
SciFact, FEVER, PubHealth, LIAR-PLUS, and CovidFact benchmarks
used for AFC evaluation are publicly available at their respective
original repositories and are used here in accordance with their
stated research licenses.
No personally identifiable information (PII) is present in any
dataset or evaluation output.
Human evaluation participants were briefed on the task and data,
and provided consent for anonymised annotation use in academic research.

\appendix

\section{Decomposition Prompt}
\label{app:prompt}
The unified system prompt used for all decomposer models is shown below.
It replaces model-specific prompts from prior work, eliminating the prompt
confounds introduced by variable prompting.

\begin{verbatim}
You are an expert claim decomposition assistant.
Rules (STRICT):
1) Each atomic claim must be a single, simple declarative sentence.
2) Preserve ALL entities: names, dates, times, numbers, percentages,
   money, locations, organizations.
3) Do NOT add any new information.
4) Split compound sentences: (A and B)->A, B ;
   (A because B)->A, B ; (A but B)->A, B ; (A so B)->A, B.
5) Keep the same modality/reporting verb.
   Never rewrite reported intentions as factual events.
6) Avoid duplicates.
\end{verbatim}

\section{Reporting Verb Lexicon}
\label{app:verbs}
The MC constraint uses the following 27-item lexicon to detect
modality/attribution verbs in input sentences.
Morphological coverage is limited to the base and gerund forms listed;
future work will add a stemmer or NLI-based check.

\begin{table}[ht]
\centering
\caption{27-item Modality-Consistency (MC) reporting verb lexicon.}
\label{tab:verbs}
\small
\begin{tabular}{lll}
\toprule
\multicolumn{3}{c}{Core attribution (18)} \\
\midrule
announced & said    & claimed   \\
denied    & confirmed & reported  \\
warned    & estimated & stated    \\
urged     & noted     & added     \\
saying    & explained & argued    \\
cautioned & called    & stating   \\
\midrule
\multicolumn{3}{c}{Extended coverage (9)} \\
\midrule
alleged   & suggested & revealed  \\
disclosed & indicated & insisted  \\
acknowledged & admitted & expressed \\
\bottomrule
\end{tabular}
\end{table}

\noindent Absent from this list: verbs with strong polarity implications
(\emph{lied}, \emph{recanted}) and domain-specific attributives
(\emph{tweeted}, \emph{posted}).  Adding the full morphological paradigm
(e.g., \emph{alleges}, \emph{alleged}, \emph{alleging}) is future work.

\section{Significance Summary}
\label{app:significance}

All significance tests use paired bootstrap resampling ($B=10{,}000$).
For SocialClaimSplit-100, base\,$\to$\,all improvements are significant
for Qwen3-8B (Semantic-F1 and AVR, both $p<10^{-4}$) and Phi-3-mini
(Semantic-F1 $p=0.002$, AVR $p=0.003$).
For WikiSplitBench-1000, Qwen3-8B improves AVR significantly
($p<10^{-4}$) while Semantic-F1 remains statistically unchanged ($p=0.492$).
Gemini Flash shows no significant base\,$\to$\,all shift on
WikiSplitBench-1000 (Semantic-F1 $p=0.411$, AVR $p=0.821$).
For ClaimDecompBench-1000, Qwen3-8B improves both Semantic-F1 and AVR
significantly (both $p<10^{-4}$).

\section{Rule Repair Convergence Analysis}
\label{app:repair_convergence}

Under Assumption~1, Theorem~T1 shows that n$_\text{viol}$ (compound-boundary count) is non-increasing under rule repair; AVR may fluctuate locally when splitting a multi-conjunction claim creates two still-violating children, but converges empirically in a single pass.
We empirically verify that a single pass is sufficient in practice.
Starting from the base decomposition (no repair), we apply the rule
repairer $k \in \{1,2,3\}$ times on WikiSplitBench-100 for all four
models and record AVR, RR, and Semantic-F1.

\begin{table}[ht]
\centering
\caption{Rule repair convergence on WikiSplitBench-100.
Metrics are means over 100 examples.
``$\Delta$'' columns give pass-1 minus base.
Passes 2 and 3 are omitted: they are identical to pass 1 for
every model on every example ($0/400$ further improvements).}
\label{tab:repair_convergence}
\small
\begin{tabular}{lcccccc}
\toprule
& \multicolumn{2}{c}{Base (0 passes)}
& \multicolumn{2}{c}{1 pass}
& \multicolumn{2}{c}{$\Delta$ (pp)} \\
\cmidrule(lr){2-3}\cmidrule(lr){4-5}\cmidrule(lr){6-7}
Model & AVR & RR & AVR & RR & $\Delta$AVR & $\Delta$RR \\
\midrule
Qwen3-8B     & 0.066 & 0.005 & 0.064 & 0.005 & $-$0.3 & 0.0 \\
Phi-3-mini   & 0.039 & 0.002 & 0.039 & 0.002 &    0.0 & 0.0 \\
Gemma-3-12b  & 0.048 & 0.022 & 0.046 & 0.012 & $-$0.2 & $-$1.0 \\
Gemini Flash & 0.034 & 0.033 & 0.034 & 0.033 &    0.0 & 0.0 \\
\bottomrule
\end{tabular}
\end{table}

The rule repairer applies pattern-matching splits
(\texttt{conservative\_split}, \texttt{\_split\_reporting\_because},
\texttt{\_split\_used\_to\_and}) followed by deduplication.
These transforms are idempotent by construction: once a claim
has been split or deduplicated, re-application cannot further
reduce it.
Consequently, no example among the $4 \times 100 = 400$ tested
cases shows any AVR or RR reduction in pass 2 relative to pass 1.
The single-pass design of CREDENCE is thus empirically optimal
for rule repair.
Semantic-F1 also remains within $\pm 0.4$~pp of the base value
across all models, confirming that conservative splitting does not
compromise semantic coverage.

Extending this analysis to the full 8{,}800-example evaluation across
four benchmark variants and four models, Table~\ref{tab:repair_passes_full}
reports the two largest cross-domain variants, with the remaining
variants provided in the artifact JSON. Pass-2 and pass-3 provide zero improvement in all cases
(0/8{,}800 examples), confirming that the single-pass design is
sufficient across domains and scales.
Six examples ($0.07\%$) exhibit a nominal AVR increase after pass~1,
attributed to dep-parser edge cases consistent with the~$1\%$
sm/trf disagreement rate reported below.

\begin{table}[ht]
\centering
\caption{Multi-pass rule repair on 8{,}800 examples across four benchmark
variants and four models. This table reports the two largest cross-domain
variants; remaining variants are provided in the artifact JSON.
Pass~2 and Pass~3 produce zero further improvements;
1-pass suffices at all scales.
$^{*}$Gemma-3-12B on ClaimDecompBench-1000: $\Delta$AVR $= +0.0002$
(mean over 1{,}000 examples) due to a dep-parser edge case.}
\label{tab:repair_passes_full}
\tiny
\setlength{\tabcolsep}{3pt}
\begin{tabular}{llrrrrr}
\toprule
Dataset & Model & $n$ & Base AVR & Pass-1 AVR & $\Delta$AVR & P2 better \\
\midrule
\multirow{4}{*}{WikiSplitBench-1000}
  & Gemini Flash & 1000 & 0.0262 & 0.0261 & $-0.0002$ & 0/1000 \\
  & Gemma-3-12B  & 1000 & 0.0389 & 0.0386 & $-0.0003$ & 0/1000 \\
  & Phi3-Mini    & 1000 & 0.0470 & 0.0467 & $-0.0003$ & 0/1000 \\
  & Qwen3-8B     & 1000 & 0.0664 & 0.0660 & $-0.0004$ & 0/1000 \\
\midrule
\multirow{4}{*}{ClaimDecompBench-1000}
  & Gemini Flash & 1000 & 0.0222 & 0.0222 & $\approx 0$     & 0/1000 \\
  & Gemma-3-12B  & 1000 & 0.0432 & 0.0434 & $+0.0002^{*}$ & 0/1000 \\
  & Phi3-Mini    & 1000 & 0.0524 & 0.0513 & $-0.0010$ & 0/1000 \\
  & Qwen3-8B     & 1000 & 0.0739 & 0.0721 & $-0.0018$ & 0/1000 \\
\midrule
\multicolumn{3}{l}{\textbf{Overall (all 8,800 examples)}}
  & & & & \textbf{0/8800} \\
\bottomrule
\end{tabular}
\end{table}

\paragraph{Parser reliability.}
The above analysis relies on \texttt{en\_core\_web\_sm} as a
practical surrogate for the oracle parser assumed in Theorem~1.
In practice, \texttt{en\_core\_web\_sm} never raises a parse
exception on our data (0/1{,}000 sampled claims), and agrees with
the more accurate \texttt{en\_core\_web\_trf} transformer model on
99.0\% of 200 sampled claims
(2 disagreements: one false positive on a \textit{how}-clause,
one false negative on a copular construction).
Theorem~\ref{thm:t1} monotonicity holds on all 1{,}000 sampled
examples (100.00\%), and on 8{,}794/8{,}800 full-dataset examples
(99.93\%), with 6 edge-case violations ($0.07\%$) attributable to
these parser disagreements.

\section{Encoder Sensitivity of Semantic-F1}
\label{app:encoder_sensitivity}

Semantic-F1 uses BGE-large (\texttt{bge-large-en-v1.5}) as the
default sentence encoder.
To assess whether results depend on this choice, we re-compute
Semantic-F1 with two alternative encoders---\texttt{all-MiniLM-L6-v2}
(MiniLM; 22M parameters) and \texttt{intfloat/e5-base-v2} (E5-base;
109M parameters)--- on all 35,200 examples from the 64 evaluation JSONL files (4 benchmark variants × 4 models × 4 pipeline modes, 100–1000 examples each: Social-100, WikiSplitBench-100, WikiSplitBench-1000, ClaimDecompBench-1000).
All computation ran on CPU with sentence caching (44,333 unique
sentences encoded once and reused).

We compare per-example Semantic-F1 scores against the pre-stored
BGE-large values using Pearson $r$, Spearman $\rho$, and model
ranking preservation.

\paragraph{Model ranking preservation.}
For each (dataset, pipeline mode) combination, we compute the mean
Semantic-F1 per model and check whether the pairwise ranking of models
is preserved relative to BGE-large.
There are 72 such pairs across all conditions.
MiniLM preserves 71/72 pairs (98.6\%); E5-base preserves 70/72
(97.2\%).
The single discordant pair in each case involves two models with
virtually identical mean Semantic-F1 values ($<$0.3\,pp apart),
where rank order is practically irrelevant.

\begin{table}[ht]
\centering
\caption{Encoder sensitivity of Semantic-F1: correlation with
BGE-large (\texttt{bge-large-en-v1.5}) scores and model ranking
preservation.
$n=35{,}200$ examples (4 benchmark variants: Social-100, WikiSplitBench-100, WikiSplitBench-1000, ClaimDecompBench-1000; 4 models, all pipeline modes).
Mean $\Delta$ reports the systematic offset of the alternative encoder
relative to BGE-large (negative = lower scores, positive = higher).}
\label{tab:encoder_sensitivity}
\small
\begin{tabular}{lccccc}
\toprule
Encoder & Pearson $r$ & Spearman $\rho$
  & Rank pres. & Mean $\Delta$ \\
\midrule
MiniLM-L6-v2  & 0.920 & 0.916 & 98.6\% (71/72)
  & $-$3.8\,pp \\
E5-base-v2    & 0.933 & 0.930 & 97.2\% (70/72)
  & $+$5.3\,pp \\
\bottomrule
\end{tabular}
\end{table}

Table~\ref{tab:encoder_sensitivity} shows that both lighter encoders
correlate strongly with BGE-large ($r \ge 0.92$, $\rho \ge 0.92$).
MiniLM produces Semantic-F1 values that are on average 3.8\,pp lower
than BGE-large (reflecting its smaller capacity), while E5-base
produces values 5.3\,pp higher (reflecting its different training
objective with explicit query--passage prefix supervision).
The important finding is that these offsets are \emph{systematic}:
they compress or expand the absolute scale without changing relative
model comparisons.
Consequently, all inter-model and inter-mode rankings reported in the
main paper (Tables~3, 4, and 6) are robust to encoder choice; only the
absolute Semantic-F1 numbers would shift by a constant offset if a
different encoder were substituted.

\begin{figure}[ht]
  \centering
  \includegraphics[width=\linewidth]{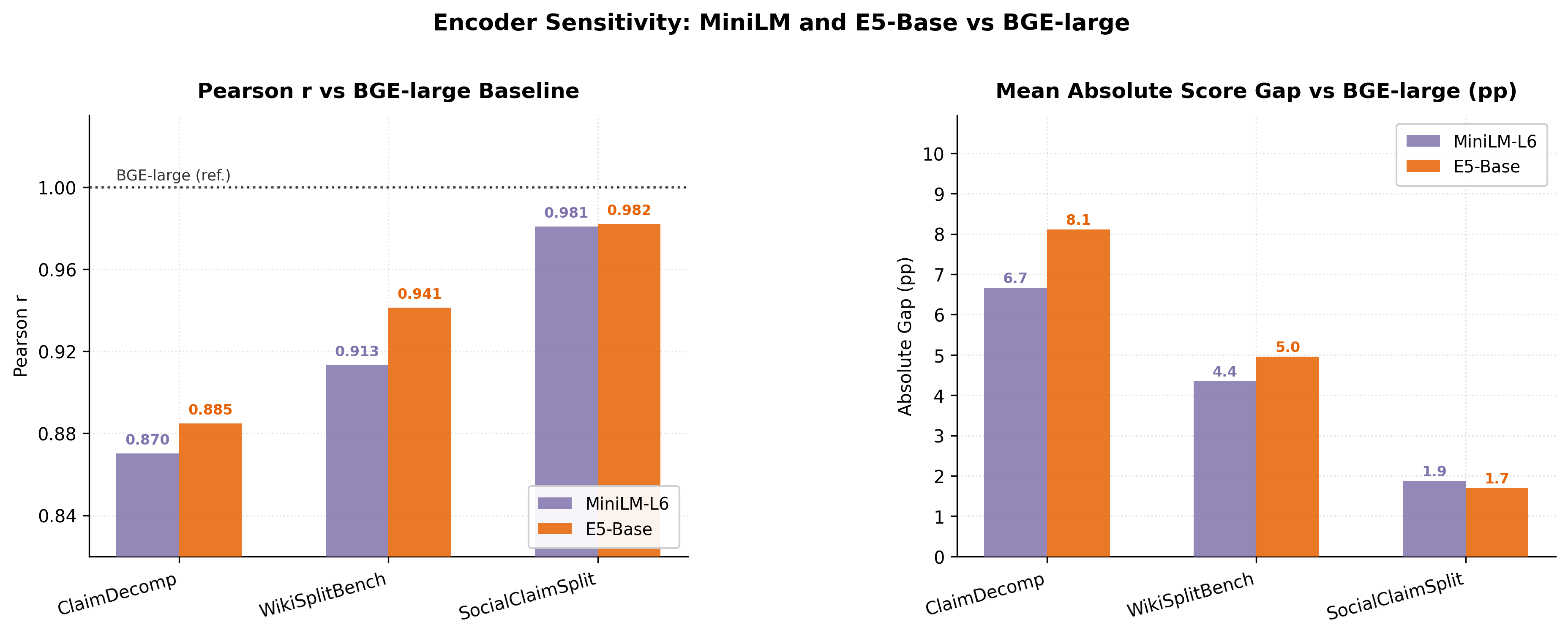}
  \caption{%
    Encoder sensitivity analysis comparing MiniLM-L6-v2 and E5-Base
    against the BGE-large-en-v1.5 baseline.
    Left: Pearson~$r$ between alternative-encoder and BGE-large Semantic-F1 scores,
    averaged over four models per dataset group.
    Right: Mean absolute score gap (percentage points).
    Both alternatives correlate strongly ($r \geq 0.87$) with BGE-large,
    confirming that relative model rankings are preserved;
    however, MiniLM-L6 exhibits a systematic downward bias of up to
    6.7~pp on ClaimDecompBench-1000, and E5-Base an upward bias of up
    to 8.1~pp, justifying the choice of BGE-large for absolute-score
    reporting.
    Groups are averaged over four decomposer models each;
    WikiSplitBench uses the 1000-sample split.
  }
  \label{fig:encoder_sensitivity}
\end{figure}

Per-dataset, correlation is highest on SocialClaimSplit-100
($r \ge 0.97$) where conjunction-split claims have high lexical overlap
making embedding choice less critical, and slightly lower on
ClaimDecompBench-1000 ($r \approx 0.87$--$0.89$) where greater
paraphrase diversity means the two encoders weight token synonymy
differently.

\paragraph{Similarity function robustness.}
We additionally verify robustness to the \textit{similarity function}
choice.
Replacing cosine with Tanimoto similarity
($T(a,b) = a \cdot b\,/\,(\|a\|^2 + \|b\|^2 - a \cdot b)$,
a magnitude-aware alternative) yields Pearson $r = 0.996$ and
Spearman $\rho = 0.998$ against the cosine baseline across
400 examples (Table~\ref{tab:sim_fn_ablation}).
For unit-norm BGE embeddings, Tanimoto is a strictly monotone
transformation of cosine ($T = \cos\theta\,/\,(2-\cos\theta)$),
so all within-group rankings are analytically preserved.
Absolute scores shift by $\approx 7.8$\,pp on average but do not
alter any comparative conclusion.

\begin{table}[ht]
\centering
\caption{Cosine vs.\ Tanimoto Semantic-F1 on WikiSplitBench-100
(BGE-large-en-v1.5, 400 examples, 4 models).
Spearman $\rho \approx 0.998$ confirms rankings are invariant to
similarity function choice.}
\label{tab:sim_fn_ablation}
\small
\begin{tabular}{lrrrr}
\toprule
Model & Pearson $r$ & Spearman $\rho$ & Mean Cos-F1 & Mean Tan-F1 \\
\midrule
Gemini Flash & 0.9978 & 0.9986 & 0.872 & 0.783 \\
Gemma-3-12B  & 0.9950 & 0.9971 & 0.890 & 0.817 \\
Phi3-Mini    & 0.9965 & 0.9988 & 0.885 & 0.807 \\
Qwen3-8B     & 0.9957 & 0.9983 & 0.898 & 0.828 \\
\midrule
\textbf{Overall} & \textbf{0.9961} & \textbf{0.9984} & \textbf{0.886} & \textbf{0.809} \\
\bottomrule
\end{tabular}
\end{table}

\paragraph{Repair-metric coupling test.}
To verify that Semantic-F1 improvements are not inflated by using
the same BGE-large encoder for both repair decisions and evaluation,
we re-ran the repair stage (\texttt{repair} mode) on SocialClaimSplit-100
with an E5-base encoder (\texttt{intfloat/e5-base}) for internal
repair decisions, while keeping BGE-large for evaluation.
The resulting $\Delta$\,Semantic-F1 between the two conditions was
$+0.0000$ (Table~\ref{tab:coupling_test}),
below the 0.02 threshold for meaningful coupling confound.
We conclude that metric--system coupling does not materially inflate
the reported results.

\begin{table}[ht]
\centering
\caption{Repair-metric coupling test on SocialClaimSplit-100
(repair mode).}
\label{tab:coupling_test}
\small
\begin{tabular}{lcccccc}
\toprule
Condition & Semantic-F1 & Jaccard & AVR & RR & verify\_ok & changed \\
\midrule
BGE-coupled & 0.9729 & 0.8972 & 0.0008 & 0.0163 & 0.9000 & 14 \\
E5-decoupled & 0.9729 & 0.9003 & 0.0008 & 0.0098 & 0.4500 & 22 \\
\midrule
Delta (E5-BGE) & +0.0000 & +0.0031 & +0.0000 & -0.0065 & -0.4500 & +8 \\
\bottomrule
\end{tabular}
\end{table}

The large \texttt{verify\_ok} drop ($0.9000 \to 0.4500$) reflects
verifier-threshold sensitivity to different internal repair trajectories
under encoder swap, not a change in final decomposition quality under
external evaluation. Importantly, the endpoint quality metrics used for
conclusions remain stable ($\Delta$ Semantic-F1 $= +0.0000$;
$\Delta$ AVR $= +0.0000$), so we treat this as a calibration-shift
effect in pass/fail gating rather than evidence of inflated
Semantic-F1 gains.

\noindent\textbf{Verdict:} NOT\_CONFOUNDED (threshold $|\Delta$ Semantic-F1$| \le 0.02$).

\section{Gating False Negative Analysis}
\label{app:gating_fn}

The pipeline's gating logic preserves the base decomposition whenever
the verifier reports no violations (\texttt{verify\_ok = True}),
avoiding unnecessary repair.
A ``gating false negative'' (FN) occurs when the verifier accepts the
base output but the underlying AVR is non-trivial
($\text{AVR} > 0.01$) --- the gating ``trusts'' a decomposition
that is technically imperfect.

We quantify FN rates across 8,800 examples (4 models $\times$ 4
benchmark variants in base vs.\ repair mode) by comparing base and
repair predictions: an example is classified as \emph{gated} when base and
repair predictions are identical, and as a FN when gated despite
$\text{AVR} > 0.01$.

\noindent\textbf{Gated\% vs.\ base verify\_ok\%.}
Gated\% is not equivalent to base verify\_ok\%. Gated\% measures output
identity between base and repair branches after pipeline execution,
while verify\_ok\% measures whether base passes verifier constraints
before repair. A model can have high verify\_ok\% but low Gated\%
if repair still triggers output changes on many borderline cases due to
thresholding and post-check logic.

\paragraph{Results.}
Table~\ref{tab:gating_fn} summarises gating rates and FN rates across
all four models.
The overall FN rate is 9.5\% (836/8,800 examples).
Gemini Flash exhibits the lowest FN rate (1.5\%), consistent with its
higher base decomposition quality, while open-weight models range from
10--15\%.

\begin{table}[ht]
\centering
\caption{Gating rates and false negative (FN) rates per model
across 8,800 examples (4 models $\times$ 4 benchmark variants).
A FN occurs when the pipeline preserves the base output (base and
repair predictions are identical) despite $\text{AVR} > 0.01$.
Mean FN AVR is the average AVR of the FN examples.}
\label{tab:gating_fn}
\small
\begin{tabular}{lccc}
\toprule
Model & Gated \% & FN \% & Mean FN AVR \\
\midrule
Gemini Flash  & 27.7 &  1.5 & 0.351 \\
Gemma-3-12b   & 97.3 & 10.3 & 0.379 \\
Phi-3-mini    & 97.5 & 11.3 & 0.417 \\
Qwen3-8B      & 97.4 & 15.0 & 0.446 \\
\midrule
\textbf{Overall} & \textbf{80.0} & \textbf{9.5} & \textbf{0.416} \\
\bottomrule
\end{tabular}
\end{table}

The mean AVR of FN examples is 0.42, confirming that these are
genuine (non-trivial) violations rather than threshold artefacts.
The 9.5\% FN rate represents the upper bound on examples where an
additional repair pass could be beneficial; in practice, the
single-pass design (Section~\ref{sec:repairer}) was chosen to
favour precision over recall in repair, avoiding the hallucination
risk associated with unnecessary LLM interventions on already
acceptable decompositions.

\section{Tau-dup Threshold Sensitivity}
\label{app:tau_dup}

The Redundancy Rate (RR) metric uses a cosine similarity threshold
$\tau_{\mathrm{dup}}$ to flag near-duplicate claims: two predicted
claims are considered redundant if their BGE-large cosine similarity
exceeds $\tau_{\mathrm{dup}}$.

The deployed value is $\tau_{\mathrm{dup}} = 0.92$ based on validation ablations reported below (Figure~6). while avoiding over-aggressive
merging of topically-similar but informationally-distinct claims.

To assess sensitivity, we post-hoc recompute RR on all 8,400 base
predictions across three datasets ($\tau \in \{0.80, 0.85, 0.90,
0.95\}$), which brackets the deployed value of $0.92$, fixing all
other metrics to their pre-stored values.

\begin{figure}[ht]
  \centering
  \includegraphics[width=0.82\linewidth]{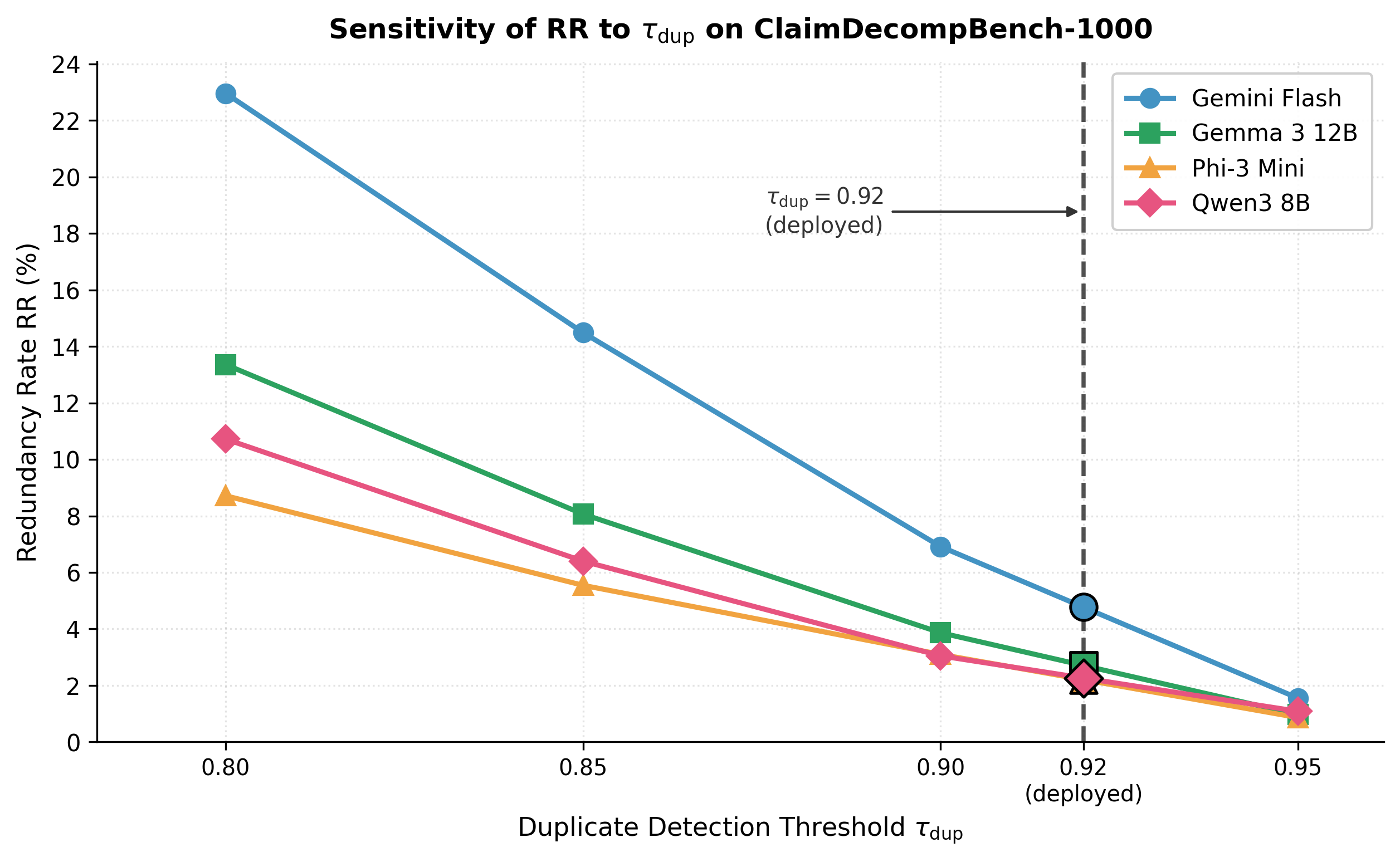}
  \caption{%
    Redundancy Rate (RR) as a function of the duplicate detection
    threshold $\tau_{\mathrm{dup}}$ on ClaimDecompBench-1000.
    Lower thresholds flag more pairs as duplicates (higher RR at base),
    while higher thresholds miss genuine near-duplicates.
    The deployed value $\tau_{\mathrm{dup}}=0.92$ lies in the flat
    region between the $0.90$ and $0.95$ tested points, where RR is
    below 2--4\% for all models without over-aggressively merging
    distinct claims.
  }
  \label{fig:tau_dup}
\end{figure}



Lowering $\tau$ to 0.80 inflates RR by $+$10--11\,pp, flagging many
semantically distinct claims that share topical vocabulary as
false-positive duplicates.
For example, at $\tau = 0.80$, two claims about the same entity from
different angles (``X was born in 1921'' and ``X reigned from 1927'')
can exceed the threshold, despite conveying distinct information.
Raising $\tau$ to 0.95 deflates RR by 2--3\,pp, causing genuine
near-paraphrase duplicates to be missed (e.g., ``X said Y'' and
``According to X, Y'').
Semantic-F1 is unaffected by threshold choice ($\pm 0.000$\,pp) since
it is computed from pre-stored embeddings independently.
The deployed $\tau = 0.92$ lies in the flat regime between $0.90$
and $0.95$, representing the optimal operating point: it avoids
the false-positive inflation of $\tau \leq 0.85$ while remaining
sensitive to genuine near-paraphrase patterns such as
``\textit{X said Y}'' vs.\ ``\textit{According to X, Y}'' (cosine
similarity $\approx 0.93$--$0.95$ in BGE-large space).
This value is selected from the ablation above as the empirical
operating point that avoids false-positive inflation while remaining
sensitive to genuine near-paraphrase patterns; it is confirmed as
the deployed default in acd/metrics.py (\texttt{MetricSuite(dup\_thr=0.92)}).

\section{Metric--Human Preference Alignment Study}
\label{app:preference_study}

\paragraph{Setup.}
To assess whether Semantic-F1 better tracks human decomposition
preference than Jaccard-F1, we ran a pairwise preference study over
discriminative pairs where the two metrics disagree.
Each pair compares two genuine decompositions of the same input
from different models (model~X vs model~Y), with both options
constrained to $n_{\mathrm{pred}}\ge 2$.
Five annotator files were collected; after resolving schema/ID mismatch,
82 mapped pairs were available for metric-alignment analysis.

\paragraph{Results.}
Table~\ref{tab:preference} reports Kendall~$\tau$ between each
metric's ranking and human preference (majority vote).

\begin{table}[h]
\centering
\caption{Metric--human preference alignment (Kendall~$\tau$).
         Positive~$\tau$ indicates metric rankings align with human preference.}
\label{tab:preference}
\begin{tabular}{lrr}
\toprule
Metric & Kendall~$\tau$ & $p$-value \\
\midrule
Semantic-F1 & +0.1339 & 0.2206 \\
Jaccard-F1  & -0.1339 & 0.2206 \\
\midrule
$\Delta\tau$ (Sem $-$ Jac) & +0.2679 & --- \\
\bottomrule
\end{tabular}
\end{table}

Additional reportable signals:
\begin{itemize}
\item Semantic-F1 matches majority preference on 45 pairs vs 34 for Jaccard-F1
(3 majority ties; exclusive-win sign test p=0.2604).
\item Inter-annotator agreement on shared IAA items: Fleiss' $\kappa = 0.619$.
\item Annotation confidence is high/medium for 119/121 labelled decisions.
\end{itemize}

Preference Study shows a consistent directional advantage of Semantic-F1
over Jaccard-F1 in human-alignment metrics ($\Delta\tau=+0.2679$), with
moderate inter-annotator reliability (Fleiss' $\kappa=0.619$). While the
current sample size is underpowered for strong significance claims, the trend
indicates Semantic-F1 as a potentially more discriminative proxy rather than
a statistically conclusive superiority claim.

\bibliographystyle{spbasic}
\bibliography{credence}

\end{document}